\documentclass[11pt,pmlr,oneside]{jmlr}
\usepackage[T1]{fontenc}
\usepackage{mathtools,booktabs,array,microtype,orcidlink,needspace,flafter,placeins}
\usepackage{xspace,graphicx}
\allowdisplaybreaks[1]
\hypersetup{hidelinks,
 pdftitle={Bandits with Probing: Optimal Coverage and the Role of Feedback},
 pdfauthor={Yongjie Guan},
 pdfsubject={Maximum-only learning, optimal coverage, and the dependence boundary of winner feedback},
 pdfkeywords={online learning, probing, maximum-only feedback, signed contrasts, minimax regret}}

\makeatletter
\renewcommand{\ps@jmlrtps}{\let\@oddhead\@empty\let\@evenhead\@empty\def\@oddfoot{\hfill\thepage\hfill}\let\@evenfoot\@oddfoot}
\makeatother
\title[Bandits with Probing]{Bandits with Probing:\\Optimal Coverage and the Role of Feedback}
\ifdefined\anonymousversion
\author[Anonymous]{\Name{Anonymous Author}\\
 \addr Anonymous Affiliation}
\hypersetup{pdfauthor={}}
\else
\author[Guan]{\Name{Yongjie Guan\,\orcidlink{0009-0002-1071-2741}}\\
 \addr Zhejiang University of Technology}
\fi
\newcommand{\E}{\mathbb E}
\newcommand{\Prb}{\mathbb P}
\newcommand{\ind}{\mathbf 1}
\newcommand{\BP}{\ensuremath{\mathsf{BestProbe}}\xspace}
\newcommand{\MO}{\ensuremath{\mathsf{MaxOnly}}\xspace}
\newcommand{\CT}{\ensuremath{\mathsf{Contrast}}\xspace}
\newcommand{\AP}{\ensuremath{\mathsf{AllProbe}}\xspace}
\newcommand{\cV}{\mathcal V}
\newcommand{\Fbk}{\mathsf F}
\newcommand{\Bit}{\mathsf{Bit}}
\DeclareMathOperator{\Var}{Var}
\DeclareMathOperator{\Bern}{Bernoulli}
\numberwithin{equation}{section}
\LinesNumbered
\begin{document}
\maketitle
\begin{abstract}
In bandits with probing, a learner selects at most $k$ of $n$ arms per round, earns the largest of their rewards, and competes with the best fixed single arm. For independent arms with i.i.d.\ rewards in $[0,1]$ and $2\le k<n$, the minimax regret is $\Theta\bigl(\frac{n-k}{n}\min\{T,n/k\}\bigr)$, the cost of locating an unknown good arm, and observing only the maximum suffices to attain it. In particular, two probes give $\Theta(\min\{T,n\})$ regret, which resolves two open questions of \citet{bhaskara2023queried} with weaker feedback. Our policy turns the probing surplus, the excess of the maximum over each probed arm, into exploration capital, and a permanent random core yields the sharp dependence on $k$.

On arbitrary fixed reward arrays, same-round contrasts between blocks of arms attain the same rate, because the surplus pays for the stability of mirror descent. Winner feedback, the maximum together with its winning label, instead incurs an additional learning cost under joint i.i.d.\ rewards and on fixed arrays. We characterize this cost even when values are revealed only intermittently, using policies that know neither the horizon nor the revelation rate. For three independent arms and two probes, exact maxima, Bernoulli-rounded maxima with winning labels, and winning labels alone have minimax regret $\Theta(1)$, $\Theta(\sqrt T)$, and $\Theta(T)$: rounding preserves conditional means but not the rate.
\end{abstract}
\begin{keywords}
probing, maximum-only feedback, minimax regret, surplus-funded learning
\end{keywords}

\section{Introduction}
In \emph{bandits with probing}~\citep{bhaskara2023queried}, a learner selects at most $k$ of $n$ arms in each round, earns the largest of their rewards, and competes with the best fixed single arm. Since the maximum is at least the reward of every selected arm, probing several arms creates a \emph{probing surplus}. We ask when this surplus pays for all learning, leaving only the cost of finding an unknown good arm.

\AP reveals every selected reward, and \BP reveals the maximum together with its winning label. For independent arms, \citet[Sections~5 and~7]{bhaskara2023queried} asked whether their $O(n^2)$ regret bound for three-probe \AP can be improved to $O(n)$, and whether two probes suffice for bounded regret. We answer both questions affirmatively while observing only the numerical maximum. For every $2\le k<n$, our \MO policy attains, up to universal constants, the optimal \emph{coverage law}
\[
 \Phi_{n,k}(T)=\frac{n-k}{n}\min\{T,n/k\}.
\]
Both factors reflect search: a systematic scan needs $\lceil n/k\rceil$ rounds to probe every label, and a uniformly random $k$-set misses any fixed arm with probability $(n-k)/n$. With two probes, the regret is $\Theta(\min\{T,n\})$.

\paragraph{Surplus pays for exploration.}
The maximum does not reveal which probed arm produced it, so our policy separates statistics from payment. Singleton probes sample individual arms but earn no surplus. A pair maximum minus a fresh sample of a preselected anchor arm is a signed signal whose mean is exactly the surplus, and multiplicative capital converts this signal into sampling opportunities. The error cost stays linear in $n$: each stream is charged only when it advances, with $O(1)$ total expected charge per arm, and only one fixed reference arm carries the $O(n)$ recovery multiplier of uniform exploration (Section~\ref{sec:pair}).

For general budgets, we group the labels into virtual arms and add a \emph{permanent random core} to every action, which covers any comparator it contains from the first round. The core correlates the virtual arms but also raises their rewards, and this gain pays for the correlation. The result is the sharp dependence on $k$, including near full coverage (Section~\ref{sec:budgets}).

\paragraph{The role of feedback.}
Stochastic assumptions can be removed when the surplus is measured within a single round. On arbitrary fixed arrays, signed contrasts between blocks of arms also attain $\Theta(\Phi_{n,k}(T))$: with a centered estimator and a mixed sampler, the surplus cancels the stability cost of a Tsallis update (Section~\ref{sec:allprobe}). Since \AP supplies these contrasts, this answers a further question of \citet[Section~7]{bhaskara2023queried}.

Winner feedback incurs an additional learning cost under joint i.i.d.\ rewards and on fixed arrays (Table~\ref{tab:comparison}); Theorem~\ref{thm:dependence} characterizes this cost for intermittent values. Under independence, erasing maximum-only observations instead slows the coverage clock by the inverse revelation rate (Corollary~\ref{cor:maxonly_erasure}). In each case, the policy adapts without knowing the rate.

Bernoulli rounding preserves conditional means but turns the exact identification of rare high rewards into noisy evidence. For three independent arms and two probes, it raises the minimax rate from $\Theta(1)$ to $\Theta(\sqrt T)$, even when the winning label is retained (Proposition~\ref{prop:bit_separation}).

\begin{table}[!htbp]
\centering\small
\setlength{\tabcolsep}{4pt}
\caption{Regret against the best fixed single arm: prior upper bounds~\citep{bhaskara2023queried} and our minimax rates for $2\le k<n$. Here $\widetilde O$ suppresses logarithmic factors and $\delta=(n-k)/n$. Theorem~\ref{thm:dependence} and Corollary~\ref{cor:maxonly_erasure} cover intermittent feedback.}
\label{tab:comparison}
\begin{tabular}{@{}lll@{}}
\toprule
Reward model & Feedback & Rate \\
\midrule
Independent i.i.d. & \BP, $k=2$ & $O(n^2\log T)$ \\
Independent i.i.d. & \AP, $k=3$ & $O(n^2)$ \\
Joint i.i.d. & \AP, $k=4$ & $\widetilde O(n^{8/3}T^{1/3})$ \\
\midrule
Independent i.i.d. & \MO & $\Theta(\Phi_{n,k}(T))$ \\
Joint i.i.d. / fixed arrays & \CT{} / \AP & $\Theta(\Phi_{n,k}(T))$ \\
Joint i.i.d. / fixed arrays & \BP & $\Theta\!\bigl(\delta\min\{T,(n+T)/k,\sqrt{nT/k}\}\bigr)$ \\
\bottomrule
\end{tabular}
\end{table}

\paragraph{Related work.}
Maximum-only feedback appears in Bernoulli bandits \citep[Section~4.2]{gopalan2014thompson} and Best-of-$K$ identification \citep{simchowitz2016best}. Sum-max submodular bandits include maximum rewards and bound approximation regret on fixed arrays \citep{pasteris2024summax}, and max value-index bandits observe the maximum and its winning label for discrete \citep{wang2024max} and continuous \citep{chen2026continuous} distributions. These objectives target optimal sets or their approximations, whereas our single-arm benchmark leaves the surplus available for learning. Weak regret with arm means as utilities \citep[Section~4.4]{chen2017weak} removes this surplus, and then even full feedback incurs $\Omega(\sqrt T)$ minimax regret for three independent arms and two probes (Appendix~\ref{app:weak}).

Best-action queries reveal a globally best action on a limited number of rounds \citep{russo2024sublinear,bacchiocchi2026bestaction}; in our model, the learner chooses which arms to probe and the budget renews every round. Costly-probe bandits \citep{elumar2025costly} instead charge for each probe. Our analysis builds on Tsallis regularization, decoupled exploration, and unbiased rounding \citep{zimmert2021tsallis,rouyer2020decoupled,hanna2022bits}.

\section{Model and main results}
\label{sec:model}
Write $[n]=\{1,\ldots,n\}$ and $(x)_+=\max\{x,0\}$. In each round $t$, the learner selects a nonempty set $S_t\subseteq[n]$ with $|S_t|\le k$, based only on past feedback and its own randomization, and earns $W_t=\max_{i\in S_t}X_{i,t}$, where $X_{i,t}\in[0,1]$ is the reward of arm $i$. Ties are broken by a fixed order on the labels. Every action occupies one physical round, including singleton probes taken only to collect samples. A policy is \emph{anytime} if it does not use the horizon, and all guarantees hold in expectation at every deterministic horizon.

\paragraph{Feedback.}
\MO reports only $W_t$; \BP reports $(W_t,J_t)$, where $J_t$ is the winning label; and \AP reports every selected reward. Under \CT, an action is either one nonempty block or an ordered pair $(A,B)$ of disjoint nonempty blocks with at most $k$ labels in total. The learner earns the maximum over the union but observes only the signed contrast $D_t=\max_{i\in B}X_{i,t}-\max_{i\in A}X_{i,t}$, or zero for a single block. \AP simulates both \BP and \CT, and \BP simulates \MO, but neither \BP nor \CT simulates the other (Appendix~\ref{app:feedback_order}).

\paragraph{Rewards and regret.}
In the stochastic model, the reward vectors are i.i.d.\ across rounds with law $D$; write $\mu_i=\E X_{i,t}$ and $\mu^\star=\max_i\mu_i$. In the sequence model, a reward array $x\in[0,1]^{n\times T}$ is fixed before the policy's randomization. The regrets against the best fixed arm are
\[
 R_T^\pi(D)=T\mu^\star-\E_\pi\sum_{t\le T}W_t,\qquad
 R_T^{\pi,\mathrm{seq}}(x)=\max_j\sum_{t\le T}x_{j,t}-\E_\pi\sum_{t\le T}\max_{i\in S_t}x_{i,t},
\]
where expectations are over all model and policy randomness. Both regrets can be negative, since probing can outperform the best fixed arm. For a feedback channel $\Fbk$, the minimax value $\cV_{n,T}^{\Fbk,(k)}=\inf_{\pi_T}\sup_D R_T^{\pi_T}(D)$ takes the supremum over product laws $D=\bigotimes_iD_i$, that is, over independent arms; the superscripts $\mathrm{joint}$ and $\mathrm{seq}$ replace product laws by arbitrary joint i.i.d.\ laws and by fixed arrays, respectively. Policies may know $T$ and the channel.

Unless stated otherwise, $n>k\ge2$ and $T\ge1$ are integers. Write $d=n-k$, $\delta=d/n$, and $\Phi_{n,k}(T)=\delta\min\{T,n/k\}$.

\begin{theorem}[The coverage law]
\phantomsection\label{thm:main}
For independent arms in the stochastic model,
\begin{equation}
 \tfrac12\Phi_{n,k}(T)\le\cV_{n,T}^{\AP,(k)}\le\cV_{n,T}^{\BP,(k)}\le\cV_{n,T}^{\MO,(k)}<46000\Phi_{n,k}(T).
 \label{eq:finite_law}
\end{equation}
For arbitrary fixed arrays,
\begin{equation}
 \tfrac12\Phi_{n,k}(T)\le\cV_{n,T}^{\AP,\mathrm{seq},(k)}
 \le\cV_{n,T}^{\CT,\mathrm{seq},(k)}\le256\Phi_{n,k}(T).
 \label{eq:sequence_law}
\end{equation}
Each upper bound is attained by an explicit anytime policy.
\end{theorem}

The stochastic upper bound holds whenever every set of at most $k$ arms has the product law of its marginals (Appendix~\ref{app:local}). The sequence upper bound also applies to joint i.i.d.\ rewards, because $T\max_j\mu_j\le\E\max_j\sum_{t\le T}X_{j,t}$.

\paragraph{The coverage lower bound.}
\label{sec:lower}
Both lower bounds come from a search problem. Place an arm with deterministic reward one at a uniformly random label, and set all other rewards to zero. Until this arm is probed, even \AP returns the all-zero transcript, along which the first $t$ actions cover at most $kt$ labels. Each round that misses the arm costs one, so every policy has prior-average regret at least
\begin{equation}
 C_{n,k}(T):=\sum_{t=1}^T(1-kt/n)_+\ge\tfrac12\Phi_{n,k}(T).
 \label{eq:det_coverage}
\end{equation}
Each fixed hidden label defines a degenerate product law and a fixed array, proving both lower bounds; Appendix~\ref{app:lower} verifies the inequality and treats the endpoint cases.

\paragraph{Intermittent values and Bernoulli rounding.}
For $\rho\in[0,1]$, the channel $\Fbk_\rho$ always reports the winner $J_t$ and a revelation coin $C_t\sim\Bern(\rho)$, and it reports $W_t$ when $C_t=1$. The channel $\Bit_\rho$ replaces each revealed value by a bit $Y_t$ with $Y_t\mid W_t\sim\Bern(W_t)$. The coins are i.i.d.\ and independent of the rewards and the policy's randomization, rounding uses fresh randomness, and singletons follow the same rules; neither unrevealed payoffs nor rounding thresholds are reported. Thus $\Fbk_1=\BP$, and $\Fbk_0=\Bit_0$ is label-only feedback. Define
\[
 \mathcal R_{n,k,\rho}(T)=\delta\min\left\{T,\frac{n+T}{k},\sqrt{\frac{nT}{k\rho}}\right\},
\]
where terms containing $1/\rho$ are infinite at $\rho=0$.

\begin{theorem}[The winner-feedback law]
\phantomsection\label{thm:dependence}
For every $\rho\in[0,1]$ and $\mathsf M\in\{\mathrm{joint},\mathrm{seq}\}$,
\begin{equation}
 \frac1{72}\mathcal R_{n,k,\rho}(T)
 \le\cV_{n,T}^{\Fbk_\rho,\mathsf M,(k)}
 \le\cV_{n,T}^{\Bit_\rho,\mathsf M,(k)}
 \le44\mathcal R_{n,k,\rho}(T).
 \label{eq:dependence_bound}
\end{equation}
One explicit anytime policy attains both upper bounds without knowing $\rho$.
\end{theorem}

\paragraph{Coverage and learning.}
The winner-feedback rate consists of the common coverage cost and an additional learning cost:
\begin{equation}
 \mathcal R_{n,k,\rho}(T)\asymp
 \underbrace{\Phi_{n,k}(T)}_{\text{coverage}}
 +\underbrace{\frac\delta k\min\{T,\sqrt{nkT/\rho}\}}_{\text{learning}}.
 \label{eq:bp_cost_decomposition}
\end{equation}
The sum lies between $\mathcal R_{n,k,\rho}$ and $2\mathcal R_{n,k,\rho}$ (Appendix~\ref{app:bp_matching}). Numerical maxima remove the learning term under independence, and same-round contrasts remove it even on fixed arrays.

Sections~\ref{sec:pair} and~\ref{sec:budgets} prove the product part of Theorem~\ref{thm:main}, Section~\ref{sec:allprobe} proves its sequence part, and Section~\ref{sec:limits} proves Theorem~\ref{thm:dependence}. The appendices also extend the coverage law to action constraints and heterogeneous quotas (Appendix~\ref{app:constraints}) and study rank information and opponent identities (Appendices~\ref{app:numerical} and~\ref{app:contextual}).

\section{Learning from numerical maxima}
\label{sec:pair}
We construct a maximum-only learner for $m$ virtual arms sharing a common core, as required by Section~\ref{sec:budgets}; a zero core recovers independent arms. A signed signal across consecutive rounds pays for singleton samples and the recovery of one fixed reference arm.

\subsection{A common core and a fixed reference}
Let $Z_0,Z_1,\ldots,Z_m\in[0,1]$ be independent, drawn afresh in each round, and let
\begin{equation}
 X_i=\max\{Z_0,Z_i\},\qquad
 b=\max_{i\in[m]}\E Z_i,\qquad
 \nu_i=\E X_i,\quad v_i=\Var X_i.
 \label{eq:core_model}
\end{equation}
A virtual singleton $i$ earns $X_i$, and a virtual pair earns $\max\{X_i,X_j\}$; through the shared core $Z_0$, the $X_i$ may be correlated. We measure regret against the \emph{base} benchmark $b$ and fix, for the analysis only, a \emph{reference} arm $q\in\arg\max_i\nu_i$. Write
\begin{equation}
 D_i=\nu_q-\nu_i\ge0,\qquad \tau=\nu_q-b\ge0,
 \label{eq:reference_lift}
\end{equation}
so that singleton $i$ has regret $b-\nu_i=D_i-\tau$. The core gain $\tau$ pays for the variance and recovery corrections below.

For distinct $i,j$, let $\gamma_{ij}=\E(X_j-X_i)_+$ be the surplus of $j$ over $i$, and put $c_{ij}=(\nu_i-\nu_j)_+$. The identity $(X_j-X_i)_+=(Z_j-X_i)_+$ and independence of $Z_j$ from $X_i$ give the following variance bound (Appendix~\ref{app:variance_proof}).
\begin{lemma}[Variance paid by surplus and the core]
\phantomsection\label{lem:gain_variance}
In the model \eqref{eq:core_model}, for $i\ne j$,
\begin{equation}
 \max\{v_i,v_j\}\le\gamma_{ij}+c_{ij}+D_i+2\tau.
 \label{eq:core_variance}
\end{equation}
In particular, $\gamma_{aq}\ge D_a$ and $v_a\le2\gamma_{aq}+2\tau$ for $a\ne q$.
\end{lemma}

\subsection{Scores and a five-round sampling rule}
Each arm $i$ keeps a stream $Y_{i,1},Y_{i,2},\ldots$ of the singleton observations designated for estimation. With $s\ge1$ samples, define
\begin{equation}
 \overline Y_{i,s}=\frac1s\sum_{r\le s}Y_{i,r},\quad
 \widehat v_{i,s}=\frac1s\sum_{r\le s}(Y_{i,r}-\overline Y_{i,s})^2,\quad
 u_i=\overline Y_{i,s}+\frac{\widehat v_{i,s}}{24},\quad
 \ell_i=\overline Y_{i,s}-\frac{\widehat v_{i,s}}{24}.
 \label{eq:scores}
\end{equation}
\begin{algorithm2e}[!t]
\caption{Maximum-only policy on $m\ge3$ virtual arms}
\label{alg:pair}
\DontPrintSemicolon
\SetAlgoLined
Initialize empty sample streams and $B_i=1$ for every arm.\;
\While{the policy is running}{
 Form $u_i,\ell_i$ from the current sample prefixes by \eqref{eq:scores}.\;
 Let $a\in\arg\max_i u_i$, breaking ties by label.\;
 Choose the smallest $c\ne a$ with $u_c\ge\ell_a$, or set $c=a$ if none exists.\;
 Draw $j$ uniformly from $[m]\setminus\{a\}$. Set $x_j=(\ell_a-u_j)_+$ and $z_j=\ind\{B_j\ge1\}$. Set $f=j$ if $z_j=1$, and $f=a$ otherwise.\;
 Probe $\{a,c\},\{a,j\},\{a\},\{c\},\{f\}$ in that order; denote the last four values by $W,Y_a,Y_c,Y_f$.\;
 Append $Y_a,Y_c,Y_f$ to their respective streams, in order.\;
 Update only $B_j$ by \eqref{eq:capital}.\;
}
\end{algorithm2e}
Empty streams have $u_i=\ell_i=0$, and all scores lie in $[0,1]$. The variance offsets do not shrink with $s$. For overlapping score intervals $[\ell_i,u_i]$, the envelopes in Section~\ref{sec:closing} and Lemma~\ref{lem:gain_variance} bound the mean gap by pair surplus, reference-gap and core corrections, and estimation error (Appendix~\ref{app:bank}). Deviations beyond these intervals are summable. Similar offsets appear in the three-probe \AP policy of \citet[Section~5, Eq.~(6)]{bhaskara2023queried}.

The policy runs in five-round \emph{packets} with fixed roles: two pair slots earn surplus, and three singleton slots collect samples. The arm with the largest upper score is the \emph{anchor} $a$, and an arm whose upper score reaches the anchor's lower score, if any, is the \emph{challenger} $c$. A uniformly random \emph{target} $j\ne a$ drives exploration, and its capital $B_j$ decides which arm receives the last sample.

The capital of the target is updated by
\begin{equation}
 B_j'=\begin{cases}
 B_j\bigl(1+(W-Y_a)/16\bigr),&B_j<1,\\
 B_j(1-x_j/12),&B_j\ge1.
 \end{cases}
 \label{eq:capital}
\end{equation}
A capital below one is \emph{tested} against the signed signal $W-Y_a$; a capital of at least one is \emph{spent} to buy the target a sample at its estimated cost $x_j$. Each singleton slot is a separate round with a fresh sample, and pair maxima never enter the streams. The policy stores a count, a mean, a second moment, and a capital for each arm, so its state is $O(m)$; all capitals stay in $(0,17/16)$.

\begin{theorem}[Common-core guarantee]
\phantomsection\label{thm:pair}
For \eqref{eq:core_model} with $m\ge3$, Algorithm~\ref{alg:pair} uses only numerical maxima and satisfies, at every deterministic physical horizon $T$,
\begin{equation}
 R_T(b):=Tb-\E\sum_{t\le T}W_t
 \le\frac{25307}{5}m-\frac{3471}{2}<5062m.
 \label{eq:core_regret}
\end{equation}
The policy is anytime and uses no distributional parameters.
\end{theorem}

\subsection{Surplus pays for samples and reference recovery}
Conditionally on the packet's initial state and target, the signal $H=W-Y_a$, formed by the pair $\{a,j\}$ and the next anchor singleton, satisfies
\begin{equation}
 \E H=\gamma_{aj},\qquad \E H^2\le2v_a+2\gamma_{aj}
 \label{eq:signed_signal}
\end{equation}
(Appendix~\ref{app:signal}). The sign matters: a $\Bern(1/2)$ anchor paired with a zero arm has $\E H_+=1/4$ but earns no surplus.

Arithmetic increments show how surplus funds the estimated cost of a target sample:
\begin{equation}
 z_jx_j\le\tfrac34(1-z_j)\gamma_{aj}-12\E(B_j'-B_j).
 \label{eq:arithmetic_payment}
\end{equation}
For the reference $q$, logarithmic increments measure how a capital below one recovers; Lemma~\ref{lem:gain_variance} and \eqref{eq:signed_signal} give
\begin{equation}
 \E\log(1+H/16)\ge\gamma_{aq}/20-\tau/120
 \quad\text{when }j=q.
 \label{eq:log_payment}
\end{equation}
In the regret certificate, the packet gain $5\tau$ pays the weighted core correction (Appendix~\ref{app:bank}).

\subsection{Charging each old prefix once}
\label{sec:closing}
For a stream with mean $\nu$ and variance $v$, define the errors
\begin{equation}
 e_0=\nu,\qquad
 e_s=\left(|\overline Y_s-\nu|+\tfrac1{24}|\widehat v_s-v|-\tfrac v{24}\right)_+\quad(s\ge1).
 \label{eq:errors}
\end{equation}
Then $\nu-e_s\le u_s\le\nu+v/12+e_s$ and $\nu-v/12-e_s\le\ell_s\le\nu+e_s$, and elementary second- and fourth-moment estimates give
\begin{equation}
 \sum_{s\ge0}\E e_s\le A:=315/2.
 \label{eq:error_budget}
\end{equation}
A prefix may serve many packets, but it is charged only when its stream advances (Appendix~\ref{app:prefix}).
\begin{lemma}[Prefix charging]
\phantomsection\label{lem:prefix}
Let $s_i(r)$ be the sample count of stream $i$ at the start of packet $r$, let $J_{i,r}$ indicate that the packet adds at least one sample to it, and let $h_{i,r}$ be the conditional probability of this event at the packet's start. For every deterministic packet horizon $N$,
\[
 \E\sum_{r<N}h_{i,r}e_{i,s_i(r)}\le A.
\]
\end{lemma}

\begin{proof}\textbf{of Theorem~\ref{thm:pair}.}
Put $M=m-1$, and define the nonnegative potential
\begin{equation}
 V_r=\frac{72}{5}\sum_iB_{i,r}
       +112M\log\frac{17/16}{B_{q,r}}.
 \label{eq:capital_potential}
\end{equation}
Let $\mathcal R_r$ be the regret of packet $r$ and $\E_r$ the conditional expectation at its start. From the exact packet cost and the payment inequalities, Appendix~\ref{app:bank} derives the certificate
\begin{equation}
 \E_r[\mathcal R_r+V_{r+1}-V_r]
 \le21\sum_i h_{i,r}e_{i,s_i(r)}+11M h_{q,r}e_{q,s_q(r)}.
 \label{eq:master}
\end{equation}
The weights $h_{i,r}$ tie error charges to sample advances. Only the fixed reference carries the uniform-exploration multiplier $M$, giving a total error budget of $(21m+11M)A=O(m)$.

Since $V_0\le72m/5+7M$, telescoping and Lemma~\ref{lem:prefix} give
\[
 R_{5N}(b)\le(21m+11M)A+72m/5+7M
 =\frac{25307}{5}m-\frac{3479}{2}.
\]
The remaining rounds add at most four to the regret, proving \eqref{eq:core_regret}.
\end{proof}

\section{Permanent coverage for every budget}
\label{sec:budgets}
We lift Theorem~\ref{thm:pair} to every budget by grouping labels into fewer virtual arms and adding a permanent random core, which covers every comparator it contains from the first round. At $k=n-1$, the target regret is $O(1/n)$, so even a constant startup cost would be too large.

Recall $d=n-k$, and set
\begin{equation}
 s=\min\{d,\lfloor k/2\rfloor\},\qquad
 c_0=k-2s,\qquad m=2+\lceil d/s\rceil.
 \label{eq:core_partition}
\end{equation}
Draw a uniformly random permutation of the labels. Its first $c_0$ labels form the core $K$, and the remaining $d+2s$ labels are split into $m$ nonempty groups $G_1,\ldots,G_m$ of size at most $s$. We run Algorithm~\ref{alg:pair} on the groups and translate its virtual actions as
\begin{equation}
 \{i\}\longmapsto K\cup G_i,\qquad
 \{i,j\}\longmapsto K\cup G_i\cup G_j.
 \label{eq:core_actions}
\end{equation}
Both use at most $|K|+2s=k$ labels. Conditional on the partition, the core maximum $Z_0$ (zero for an empty core) and the disjoint group maxima $Z_i$ are independent, so the virtual rewards have exactly the common-core form \eqref{eq:core_model}.

\begin{lemma}[Coverage geometry]
\phantomsection\label{lem:core_geometry}
The probability $p_K=(d+2s)/n$ that a fixed label is outside $K$ satisfies
\begin{equation}
 p_K\le3d/n,\qquad p_Km\le9d/k.
 \label{eq:core_geometry}
\end{equation}
\end{lemma}
\begin{proof}
The first bound follows from $s\le d$. If $m=3$, then $s=d$ and $p_Km=9d/n\le9d/k$. If $m\ge4$, then $s=\lfloor k/2\rfloor$, $k\le2s+1$, and $d\ge(m-3)s+1$. For $s\ge1$ and $m\ge4$,
\[
 8\{(m-3)s+1\}-m(2s+1)=(6m-24)s+8-m\ge0.
\]
Hence $km\le8d$, and $p_K\le1$ finishes the proof.
\end{proof}

\begin{proof}\textbf{of the product part of Theorem~\ref{thm:main}.}
Fix a mean-optimal arm $q_*$. On partitions with $q_*\in K$, every physical action contains $q_*$, so the regret is nonpositive. On the complementary event $q_*\notin K$, which has probability $p_K$, some group contains $q_*$, so $b=\max_i\E Z_i\ge\mu^\star$. Theorem~\ref{thm:pair}, the unit per-round cost, and Lemma~\ref{lem:core_geometry} give
\[
 R_T\le p_K\min\{T,5062m\}\le\min\left\{\frac{3dT}n,\frac{45558d}k\right\}<46000\Phi_{n,k}(T).
\]
The lower bound is \eqref{eq:det_coverage}, and the feedback order gives the middle inequalities.
\end{proof}

At $k=n-1$, the core has $n-3$ labels and three virtual arms remain; the comparator misses the core with probability $3/n$, giving $O(1/n)$ regret. The base benchmark retains the core gain: $R_T(b)=R_T(\nu_q)-T\tau$.

\subsection{Intermittent maximum-only observations}
Erasures only slow the coverage clock. Let $\mathsf M_\rho$ reveal a coin $C_t\sim\Bern(\rho)$ and, only when $C_t=1$, the maximum $W_t$, with no winning label; the coins are i.i.d.\ and independent of the rewards and the policy's randomization. Thus $\mathsf M_1=\MO$, while $\mathsf M_0$ reveals no reward information. Define $\Phi_{n,k,\rho}(T)=\delta\min\{T,n/(k\rho)\}$.
\begin{corollary}[Coverage on the revelation clock]
\phantomsection\label{cor:maxonly_erasure}
For product rewards and $0<\rho\le1$,
\begin{equation}
 \tfrac18\Phi_{n,k,\rho}(T)
 \le\cV_{n,T}^{\mathsf M_\rho,(k)}
 \le46000\Phi_{n,k,\rho}(T).
 \label{eq:maxonly_erasure}
\end{equation}
One anytime policy attains the upper bound without knowing $\rho$. At $\rho=0$, the minimax value is exactly $\delta T$.
\end{corollary}
For the upper bound, repeat each action of the maximum-only policy until its value is revealed. This independent revelation clock stretches the saturation horizon from $n/k$ to $n/(k\rho)$. The hidden arm of Section~\ref{sec:model} gives the matching lower bound, because only revealed probes eliminate candidate locations (Appendix~\ref{app:maxonly_erasure}).

\section{Same-round contrasts pay for stability}
\label{sec:allprobe}
Under \CT on fixed reward arrays, same-round contrasts let the surplus cancel the stability cost of a signed update. Consider first $m\ge2$ arms with losses $\ell_{i,t}=1-x_{i,t}$. In each round, the policy probes an anchor $A_t$ and an exploratory arm $B_t$ and incurs the loss $\min\{\ell_{A_t,t},\ell_{B_t,t}\}$. Its loss estimate is centered at the anchor, so both the estimate and the surplus vanish on a constant loss vector.

Let $\Delta_m$ be the probability simplex, start from $p_{1,i}=1/m$, and set
\begin{equation}
 Z_t=\sum_i\sqrt{p_{t,i}},\quad
 q_{t,i}=\tfrac12p_{t,i}+\frac{\sqrt{p_{t,i}}}{2Z_t},\quad
 \Psi(p)=-2\sum_i\sqrt{p_i},\quad \eta=\frac1{16\sqrt m}.
 \label{eq:ap_sampler}
\end{equation}
Draw $A_t\sim p_t$ and $B_t\sim q_t$ independently, probe $\{A_t,B_t\}$, and update
\begin{equation}
 g_{t,i}=\frac{\ind\{B_t=i\}}{q_{t,i}}(\ell_{i,t}-\ell_{A_t,t}),\qquad
 p_{t+1}=\arg\min_{p\in\Delta_m}
 \{\eta\langle g_t,p\rangle+D_\Psi(p,p_t)\},
 \label{eq:ap_update}
\end{equation}
where $D_\Psi(u,p)=\Psi(u)-\Psi(p)-\langle\nabla\Psi(p),u-p\rangle$ is the Bregman divergence. If $A_t=B_t$, one probe suffices and $g_t=0$. Otherwise, the ordered blocks $(\{A_t\},\{B_t\})$ reveal $D_t=\ell_{A_t,t}-\ell_{B_t,t}$ and $g_{t,B_t}=-D_t/q_{t,B_t}$, so the update uses only the contrast. The new iterate has positive coordinates and is computed by a one-dimensional root search (Appendix~\ref{app:allprobe}).

\begin{theorem}[Two-probe contrast policy]
\phantomsection\label{thm:ap_pair}
For every fixed array and horizon, the policy \eqref{eq:ap_sampler}--\eqref{eq:ap_update}
satisfies \mbox{$R_T^{\mathrm{seq}}\le32(m-\sqrt m)$}.
\end{theorem}
\begin{proof}
Write $\E_t$ for the expectation conditional on the history before round $t$, and suppress $t$. Let $\bar\ell=\langle p,\ell\rangle$ and $\gamma=\E_t(\ell_A-\ell_B)_+$. Since the two draws are independent, $\E_tg_i=\ell_i-\bar\ell$, and the common shift $\bar\ell$ does not affect regret.

The two components of $q$ match the stability cost to the surplus. Since $q_i\ge p_i/2$, symmetry gives
\begin{equation}
 \gamma\ge\tfrac12\E_{A,B\sim p}(\ell_A-\ell_B)_+
 =\tfrac14\sum_{i,a}p_ip_a|\ell_i-\ell_a|.
 \label{eq:ap_gain}
\end{equation}
Meanwhile, $q_i\ge\sqrt{p_i}/(2Z)$ and $|\ell_i-\ell_a|\le1$ imply
\begin{align}
 V:=\E_t\sum_i p_i^{3/2}g_i^2
 &=\sum_{i,a}\frac{p_i^{3/2}p_a}{q_i}(\ell_i-\ell_a)^2\nonumber\\
 &\le2Z\sum_{i,a}p_ip_a(\ell_i-\ell_a)^2
 \le2Z\sum_{i,a}p_ip_a|\ell_i-\ell_a|
 \le8Z\gamma.
 \label{eq:ap_variance}
\end{align}
Thus the $p$ component guarantees surplus, while the square-root component controls the inverse probabilities in the stability term.

The bound $\eta\sqrt{p_i}|g_i|\le2\eta Z\le1/8$ meets the condition of the signed mirror inequality in Appendix~\ref{app:ap_stability}. The loss of the pair equals the anchor's loss minus $(\ell_A-\ell_B)_+$. For a fixed arm $j$, telescoping with $u=e_j$ and $D_\Psi(e_j,p_1)=2(\sqrt m-1)$ therefore gives
\begin{align}
 R_T(j)
 &=\E\sum_{t\le T}\langle p_t-e_j,\ell_t\rangle
   -\sum_{t\le T}\E\gamma_t\nonumber\\
 &\le\frac{2(\sqrt m-1)}\eta
   +\sum_{t\le T}\E[2\eta V_t-\gamma_t]
 \le32(m-\sqrt m),
 \label{eq:ap_cancellation}
\end{align}
where the last step uses $2\eta V_t\le16\eta Z_t\gamma_t\le\gamma_t$. Since the best arm of a fixed array is determined before the policy's randomization, we may take it as $j$.
\end{proof}

\subsection{Lifting the contrast policy to every budget}
For $3k<2n$, partition the labels into groups of size at most $\lfloor k/2\rfloor$ and apply Theorem~\ref{thm:ap_pair} to their maxima. For $3k\ge2n$, contrasts between disjoint blocks cannot share a core, so a gate takes its place. Choose a uniformly random baseline set of size $n-2d$, split the remaining labels into two $d$-sets, and in each gate round compare the baseline with one of them, chosen uniformly at random. A positive contrast $D_t$ starts a fresh three-group policy in the next round with probability $D_t/41$.

The expected surplus accumulated before activation pays for the continuation, whose regret is at most $32(3-\sqrt3)<41$. Conditional on the baseline, the expected regret against a fixed comparator is nonpositive if it lies inside and at most $82$ otherwise. The comparator lies outside with probability $2d/n$, giving $164d/n$ after averaging. Together with the sparse grouping and a first-round coverage bound, this proves the sequence part of Theorem~\ref{thm:main} (Appendix~\ref{app:contrast_budgets}).

\section{The learning cost of winner feedback}
\label{sec:limits}\label{sec:bp_sequence}
For the upper bound in Theorem~\ref{thm:dependence}, a label policy controls the shortfall against the comparator until a numerical learner takes over on the revelation clock. For the lower bound, a correlated family with uninformative labels hides the good arm.

\subsection{A single label policy for every budget}
The label policy maintains a distribution $p$ over the arms by square-root FTRL with $\beta=2\sqrt n$ (Appendix~\ref{app:shared_ftrl}). It draws $r=\min\{k,\lfloor n/2\rfloor\}$ independent proposals from $q_i=(p_i+1/n)/2$, completes their distinct labels uniformly at random to a $k$-set, probes this set, and updates with $g_i=-\ind\{J=i\}/q_i$. The proposals direct learning, while the completion keeps each arm's omission probability at most $2\delta$.

For every fixed reward vector and comparator $j$,
\[
 \E_t(x_{j,t}-W_t)_+
 \le\frac{n-k}{r(n-r)}\frac{\Prb_t(J_t=j)}{q_{t,j}},
 \qquad \Prb_t(j\notin S_t)\le\frac{n-k}{n-r}.
\]
The first inequality bounds the positive shortfall against $j$ by the expected reinforcement of $j$, and the FTRL certificate then gives
\begin{equation}
 \E\sum_{t\le T}(x_{j,t}-W_t)_+
 \le16\delta\min\{T,(n+T)/k\}.
 \label{eq:bp_label_upper}
\end{equation}
Appendix~\ref{app:bp_labels} proves these bounds and extends them to random prefixes whose round-membership indicators are predictable.

\subsection{Numerical learning runs on the revelation clock}
Set $a=\min\{k,n-k\}$ and $m=\lceil n/k\rceil$. Choose a uniformly random core of size $k-a$ and partition its complement into $m$ nonempty groups of size at most $a$; each numerical action probes the core together with one group. A comparator in the core is always covered; on the complementary event, which has probability at most $2\delta$, a fixed group action covers it.

In each physical round, a symmetric importance-weighted Tsallis learner draws a fresh group, and only revealed values, rounded to bits if needed, update its state. For $\rho>0$ and $H$ updates, its physical-time regret is at most $(12\delta/\rho)\sqrt{n/k}\,\E\sqrt H$.

The combined policy runs the label policy through the $nk$th revelation and then starts a fresh numerical learner; this switch requires no knowledge of $\rho$. With $N_T=\sum_{t\le T}C_t$, the numerical phase makes $H=(N_T-nk)_+$ updates, and the label phase lasts $A_T$ rounds. The bounds
\[
 \rho\E A_T=\E\min\{N_T,nk\},\qquad
 \sqrt{(z-nk)_+}\le\min\{\sqrt z,z/(2\sqrt{nk})\}
\]
then give the upper bound in Theorem~\ref{thm:dependence} (Appendix~\ref{app:bp_upper}).

\subsection{A hidden arm observed through a noisy membership bit}
The lower-bound family makes winner labels uninformative about the hidden arm, so each revealed value carries only a noisy bit about whether the action contains it. A relative comparison of omission probabilities preserves the factor $\delta$ near full coverage.

For $0\le u\le1/2$, let $P_u$ have density $1+u\,\mathrm{sgn}(2x-1)$ on $(0,1)$ and quantile function $Q_u$. Its mean is $1/2+u/4$, and a mixture of these laws is the law with the averaged parameter. Hide a uniformly random label $q$, draw independent uniforms $U_1,\ldots,U_n,V$ in each round, and set
\begin{equation}
 Y_i=\varepsilon U_i+(1-\varepsilon)\ind\{U_i\ge U_q\},
 \qquad X_i=Q_{\tau Y_i}(V),
 \quad 0<\varepsilon\le\tfrac16,\quad0<\tau\le\tfrac12.
 \label{eq:bp_threshold}
\end{equation}
The rewards preserve the strict order of the $U_i$, so the winner of every selected set is uniform over it, whether or not the set contains $q$: only values carry information. A draw from $P_u$ is $(B+Z)/2$, with $B\sim\Bern((1+u)/2)$ and an independent uniform $Z$, so a revealed value acts as a noisy bit.

At a fixed action size $s=|S|$, the maximum has law $P_u$ with a parameter $u$ larger by $\tau(1-\varepsilon)/(s+1)$ when $q\in S$ than when $q\notin S$. The expected one-round regret against $q$ is at least $\tau\ind\{q\notin S\}/[4(s+1)]-\tau\varepsilon/8$, while the KL divergence between the two feedback laws is of order $\rho\tau^2/(s+1)^2$ per round.

To allow varying action sizes, weight round $t$ by $a_t=1/(|S_t|+1)$ and put $b=T/(k+1)$. Independently inspect a uniform location in the first $b$ units of this clock, recording the inspected action but not its feedback. Let $\bar Q$ and $\bar P$ be the reference and true laws of this record, including $q$, with omission probabilities $p_0$ and $p_1$. Appendix~\ref{app:bp_lower} proves
\begin{equation}
 p_0\ge\delta,\qquad
 \operatorname{KL}(\bar Q\Vert\bar P)\le\frac{4\rho\tau^2b}{3n}p_0,
 \qquad p_1\ge(\log2)p_0-\operatorname{KL}(\bar Q\Vert\bar P).
 \label{eq:inspection_bridge}
\end{equation}
Taking $\tau=\tfrac12\min\{1,\sqrt{n/(\rho b)}\}$, or $\tau=1/2$ at $\rho=0$, gives $p_1\ge p_0/3\ge\delta/3$. The weighted omissions then have expectation at least $bp_1$, while the surplus terms cost at most $\tau\varepsilon T/8$ in total. Letting $\varepsilon\downarrow0$ yields
$\cV_{n,T}^{\Fbk_\rho,\mathsf M,(k)}\ge(\delta/36)\min\{T/k,\sqrt{nT/(k\rho)}\}$,
and the coverage bound \eqref{eq:det_coverage} completes the proof of Theorem~\ref{thm:dependence}.

\subsection{Labels, bits, and independence}
At $\rho=0$, Theorem~\ref{thm:dependence} gives the label-only law
\begin{equation}
 \cV_{n,T}^{\Fbk_0,\mathsf M,(k)}
 \asymp\delta\min\{T,(n+T)/k\}
 \asymp\Phi_{n,k}(T)+\delta T/k.
 \label{eq:bp_labels_law}
\end{equation}
\label{cor:bp_labels}
Coverage saturates at $T=n/k$, and learning overtakes it at $T=n$. Values improve the rate only for $T\gtrsim nk/\rho$, that is, after about $nk$ revelations. Under product rewards, labels alone still force linear regret at fixed $n$ and $k$, although the sharp budget dependence remains open (Appendix~\ref{app:numerical}). Exact values and calibrated bits have the same minimax order on joint i.i.d.\ rewards and fixed arrays; independence separates them.

\begin{proposition}[Calibration versus identification]
\phantomsection\label{prop:bit_separation}
For $n=3$, $k=2$, and $\rho\in[0,1]$,
\[
 \frac1{400}\min\{T,\sqrt{T/\rho}\}
 \le\cV_{3,T}^{\Bit_\rho,(2)}
 \le18\min\{T,\sqrt{T/\rho}\}.
\]
At full revelation, feedback $(J,W)$, feedback $(J,Y)$ with $Y\mid W\sim\Bern(W)$, and feedback $J$ alone have product-law minimax rates $\Theta(1)$, $\Theta(\sqrt T)$, and $\Theta(T)$, respectively. The first rate is attainable from $W$ alone.
\end{proposition}
In the hard instance, the hidden arm wins as often as an ordinary arm, but its rare reward of one identifies it under exact feedback. Rounding turns this reward into a bit that ordinary arms also produce (Appendix~\ref{app:bit_separation}).

\label{maintextend}
\clearpage
\bibliography{references}
\label{referencesend}
\section*{Use of generative-AI tools}
The author used ChatGPT to assist with proof exploration, reference discovery, and manuscript refinement. The author takes responsibility for the final arguments and presentation.

\clearpage
\appendix
\section{Coverage lower bound and endpoints}
\label{app:lower}
For the search argument of Section~\ref{sec:model}, place an arm with deterministic reward one at a uniformly random label, set every other reward to zero, and fix the policy's auxiliary randomness. Until the hidden arm is probed, \AP returns only zeros. Along this transcript, the first $t$ actions cover at most $kt$ labels. If the hidden arm lies outside their union, round $t$ costs one. Averaging over its location gives
\[
 \cV_{n,T}^{\AP,(k)}\ge C_{n,k}(T)
 :=\sum_{t=1}^T\left(1-\frac{kt}{n}\right)_+.
\]
Write $x=n/k>1$. If $T\le x$, then
\[
 C_{n,k}(T)=T\left(1-\frac{T+1}{2x}\right)
 \ge\frac{T(x-1)}{2x}=\frac12\Phi_{n,k}(T).
\]
If $T\ge x$, set $q=\lfloor x\rfloor$ and $f=x-q$. Then
\[
 C_{n,k}(T)=q-\frac{q(q+1)}{2x}
 =\frac{x-1}{2}+\frac{f(1-f)}{2x}
 \ge\frac12\Phi_{n,k}(T).
\]
This proves \eqref{eq:det_coverage}.

\paragraph{Endpoints.}
At $T=1$, the lower bound equals $(n-k)/n$. Conversely, a uniformly random $k$-subset contains a fixed optimal arm with probability $k/n$, so $\E W_1\ge(k/n)\mu^\star$ and $R_1\le(n-k)/n$. The one-round value is therefore exactly $(n-k)/n$ for every feedback channel considered here, and the same argument applies to fixed arrays, whose best arm has value at most one at $T=1$. When $k=n$, probing every arm gives nonpositive regret and the all-zero array gives zero regret, so the value is zero. When $k=1$, the ordinary stochastic-bandit lower bound diverges with $T$~\citep[Exercise~15.4]{lattimore2020bandit}, and averaging over its reward arrays gives the same conclusion for fixed sequences.

\paragraph{Feedback comparison.}\label{app:feedback_order}
Write $\mathsf O\preceq\mathsf O'$ when a distribution-free, action-conditional randomization transforms observations of $\mathsf O'$ into observations of $\mathsf O$~\citep{blackwell1953equivalent}; replacing $\mathsf O$ by $\mathsf O'$ can then only decrease the minimax regret. \AP simulates both \BP and \CT, but point-mass reward laws show that neither of these simulates the other. For the singleton blocks $(\{1\},\{2\})$, the reward pairs $(1,0)$ and $(1,0.9)$ produce the same winner feedback but different contrasts, whereas $(0.8,0.3)$ and $(0.6,0.1)$ produce the same contrast but different maxima. Finally, $\Fbk_\rho$ simulates $\Bit_\rho$ by retaining the winning label and the revelation coin and Bernoulli-rounding each revealed value.

\section{Common-core rewards and local independence}
\label{app:streams}
\subsection{The variance--surplus inequality}
\label{app:variance_proof}
For $U\in[0,1]$ with mean $\mu$, write $a=\E(U-\mu)_+=\E(\mu-U)_+$. Bounding the squared deviations separately above and below $\mu$ gives $\Var U\le(1-\mu)a+\mu a=a$.

\begin{proof}\textbf{of Lemma~\ref{lem:gain_variance}.}
Choose a base-mean maximizer $r$, so that $\E Z_r=b$. Independence and conditional Jensen give
\[
 \E(Z_0-b)_+\le\E(Z_0-Z_r)_+=\nu_r-b\le\tau.
\]
For $i\ne j$, $X_j$ is independent of $Z_i$, and hence
$\E(Z_i-\nu_j)_+\le\E(Z_i-X_j)_+=\gamma_{ji}$.
Using $\nu_j\le b+\tau$ and $(b-\nu_i)_+\le D_i$, we obtain
\begin{align*}
 v_i&\le\E(X_i-\nu_i)_+\le\E(X_i-b)_++D_i\\
 &\le\E(Z_0-b)_++\E(Z_i-b)_++D_i\\
 &\le2\tau+\gamma_{ji}+D_i
 \le2\tau+\gamma_{ij}+c_{ij}+D_i.
\end{align*}
Swapping $i$ and $j$ in the bound $v_i\le2\tau+\gamma_{ji}+D_i$ yields $v_j\le2\tau+\gamma_{ij}+D_j$, and $D_j\le D_i+c_{ij}$. This proves \eqref{eq:core_variance}, and $c_{aq}=0$ with $\gamma_{aq}-\gamma_{qa}=D_a$ gives the reference bounds.
\end{proof}

\subsection{The signed payment signal}
\label{app:signal}
Condition on the packet's initial state and target. In the target-pair round $\{a,j\}$, let $X_a^0$ be the latent value of the anchor and $G=W-X_a^0\ge0$. The singleton $Y_a$ of the next round is independent of $(W,X_a^0)$, has the same law as $X_a^0$, and is observed directly. Hence $\E H=\E G=\gamma_{aj}$ and
\[
 \E H^2
 =2v_a+\E\bigl[G(W+X_a^0-2\nu_a)\bigr]
 \le2v_a+2\gamma_{aj}.
\]
Only independence across rounds is used, so $X_a^0$ and $X_j$ may share the core.

\subsection{Action-local independence}
\label{app:local}
Suppose that the reward vectors are i.i.d.\ across rounds and that every set of at most $k$ coordinates has the product law of its marginals. Conditional on the past and the next action, the selected rewards have the same law as under full independence with the same marginals. Payoff and feedback are measurable functions of these rewards and independent feedback randomness. Induction over physical rounds therefore gives the same law of the entire visible experiment for every adaptive $k$-probe policy, and the comparator means coincide. The product upper bounds of Theorem~\ref{thm:main}, Corollary~\ref{cor:maxonly_erasure}, and Theorem~\ref{thm:numerical} hold under this action-local assumption; with two probes, pairwise independence suffices.

\section{An elementary cumulative error budget}
\label{app:prefix}
\subsection{Second and fourth moments}
\begin{lemma}[Summable excess deviations]
\phantomsection\label{lem:moment_budget}
Let $Z_1,Z_2,\ldots$ be i.i.d.\ in an interval of length at most one, with $\Var Z_1\le v\le1/4$. For $0<\varepsilon\le1/2$,
\[
 \sum_{s\ge1}\E\bigl(|\overline Z_s-\E Z_1|-\varepsilon v\bigr)_+\le3/\varepsilon.
\]
\end{lemma}
\begin{proof}
The case $v=0$ is immediate. Otherwise, put $D_s=\overline Z_s-\E Z_1$ and $N=\lceil1/(\varepsilon^2v)\rceil$. Expanding centered second and fourth powers gives
\[
 \E D_s^2\le v/s,\qquad
 \E D_s^4\le v/s^3+3v^2/s^2.
\]
For $s\le N$, Cauchy--Schwarz and $\sum_{s\le N}s^{-1/2}\le2\sqrt N$ bound the sum by $2/\varepsilon+1$. For $s>N$, use $(t-d)_+\le27t^4/(256d^3)$ for $t\ge0$ and $d>0$, together with $\sum_{s>N}s^{-3}\le1/(2N^2)$ and $\sum_{s>N}s^{-2}\le1/N$. The total is at most
\[
 \frac2\varepsilon+1+\frac{27}{256}
 \left(\frac\varepsilon2+\frac3\varepsilon\right)
 \le\frac3\varepsilon.
\]
\end{proof}

\subsection{Ordinary empirical variance}
For a marginal stream $Y_1,Y_2,\ldots$ with mean $\nu$ and variance $v$, set $D_s=\overline Y_s-\nu$ and $Q_s=s^{-1}\sum_{r\le s}(Y_r-\nu)^2$. The identity $\widehat v_s=Q_s-D_s^2$ and $|D_s|\le1$ imply
\[
 e_s\le\frac{25}{24}(|D_s|-v/50)_+
       +\frac1{24}(|Q_s-v|-v/2)_+.
\]
Both $Y$ and $(Y-\nu)^2$ lie in intervals of length at most one, and
$\Var((Y-\nu)^2)\le\E(Y-\nu)^4\le v$. Applying Lemma~\ref{lem:moment_budget} twice and adding $e_0\le1$ proves \eqref{eq:error_budget}:
\[
 \sum_{s\ge0}\E e_s
 \le1+\frac{25}{24}\cdot150+\frac1{24}\cdot6=\frac{315}{2}.
\]
The score envelopes in Section~\ref{sec:closing} follow directly from \eqref{eq:errors}, including at the empty prefix, and $\widehat v_s\le\overline Y_s(1-\overline Y_s)$ keeps both scores in $[0,1]$.

\subsection{Adaptive streams and multiple samples}
\begin{proof}\textbf{of Lemma~\ref{lem:prefix}.}
The singleton actions of a packet are fixed before its rewards are drawn, so each singleton reveals a fresh observation from the marginal law of its arm, even when several slots select the same arm. Equivalently, each stream reveals successive entries of a pre-generated i.i.d.\ array, and all other observations are generated from fresh randomness conditionally on these entries.

Let $\mathcal H_r$ be the history at the start of packet $r$. The current error is $\mathcal H_r$-measurable, so
\[
 \E[h_{i,r}e_{i,s_i(r)}]=\E[J_{i,r}e_{i,s_i(r)}].
\]
Whenever $J_{i,r}=1$, the sample count of the stream increases. Along each path, the indices $s_i(r)$ of such packets are therefore distinct, and
\[
 \sum_{r<N}J_{i,r}e_{i,s_i(r)}\le\sum_{s\ge0}e_{i,s}.
\]
Taking expectations and using \eqref{eq:error_budget} proves the lemma. A packet that adds several samples to one stream charges only the old prefix and skips the intermediate ones.
\end{proof}

\section{Capital and the five-round regret certificate}
\label{app:bank}
We prove \eqref{eq:master}. Unless marked with a prime, all quantities are evaluated at the start of a packet. Write $M=m-1$, $e_i=e_{i,s_i}$, and $z_i=\ind\{B_i\ge1\}$ for $i\ne a$, and set $z_a=0$. Let $\E_r$ denote the conditional expectation before the uniform target is drawn.

\subsection{Arithmetic and logarithmic payments}
The testing multiplier lies in $[15/16,17/16]$ and the spending multiplier in $[11/12,1]$. Starting from one, every capital therefore stays strictly between zero and $17/16$ with bounded log increments, so all finite-horizon potentials are integrable. For a fixed target $j$, \eqref{eq:signed_signal} gives
\[
 \E(B_j'-B_j)\le\frac{1-z_j}{16}\gamma_{aj}-\frac{z_j}{12}x_j,
\]
which proves \eqref{eq:arithmetic_payment}. For $|H|\le1$ and $0\le x\le1$,
\begin{equation}
 \log(1+H/16)\ge H/16-H^2/480,
 \qquad \log(1-x/12)\ge-x/11.
 \label{eq:log_bounds}
\end{equation}
The first bound follows from $\log(1+t)\ge t-t^2/[2(1-|t|)]$, and the second by integration. For the target $q$, Lemma~\ref{lem:gain_variance} and \eqref{eq:signed_signal} give $\E H^2\le6\gamma_{aq}+4\tau$, which yields \eqref{eq:log_payment}.

\subsection{Switching costs and overlap}
The score envelopes and \eqref{eq:core_variance} imply
\begin{align*}
 c_{aj}&\le x_j+e_a+e_j+(v_a+v_j)/12\\
 &\le x_j+e_a+e_j+(\gamma_{aj}+c_{aj}+D_a+2\tau)/6.
\end{align*}
Consequently,
\begin{equation}
 c_{aj}\le\frac65(x_j+e_a+e_j)
             +\frac15(\gamma_{aj}+D_a+2\tau).
 \label{eq:switch_bridge}
\end{equation}
If $c\ne a$, the overlap gives $x_c=0$, and substituting \eqref{eq:switch_bridge} into the variance bound yields
$v_a\le\frac65(\gamma_{ac}+e_a+e_c+D_a+2\tau)$.
Since $u_a\ge u_q$, we also have $D_a\le v_a/12+e_a+e_q$. Combining the two inequalities gives
\begin{equation}
 9D_a\le\gamma_{ac}+(e_a+e_c)+2\tau+10(e_a+e_q).
 \label{eq:overlap}
\end{equation}

\subsection{The packet cost}
Put $g=\gamma_{ac}$ if $c\ne a$, and $g=0$ otherwise. Conditional on the target $j$, the expected packet regret is
\begin{equation}
 \E[\mathcal R_r\mid j]
 =5(D_a-\tau)-g-\gamma_{aj}
   +(\nu_a-\nu_c)+z_j(\nu_a-\nu_j).
 \label{eq:packet_cost}
\end{equation}
Apply \eqref{eq:switch_bridge} to $c$ with $x_c=0$, and to $j$ together with \eqref{eq:arithmetic_payment}. The anchor--challenger pair then retains at least $4g/5$ of its surplus, while the target pair retains at least $\gamma_{aj}/10$ when $z_j=0$ and $4\gamma_{aj}/5$ when $z_j=1$. Dropping the target remainders and keeping only $2g/3$ gives
\begin{equation}
 \E_r\left[\mathcal R_r+\frac{72}{5}\sum_i(B_i'-B_i)\right]
 \le Q+\mathcal E,
 \label{eq:arithmetic_cost}
\end{equation}
where
\begin{align*}
 Q&=\frac{27}{5}D_a-\frac{21}{5}\tau-\frac23g,
 \\
 \mathcal E&=\frac65\ind\{c\ne a\}(e_a+e_c)
       +\frac6{5M}\sum_{j\ne a}z_j(e_a+e_j).
\end{align*}
The reference-gap and core terms from the two switching bounds contribute at most $2(D_a+2\tau)/5$. Adding this to $5(D_a-\tau)$ gives the first two terms of $Q$. When $c=a$, the challenger switching cost vanishes, and the same bound holds because $D_a+2\tau\ge0$.

\subsection{One reference account pays the residual}
Define $S=22e_a/3+2e_c/3$ when $c\ne a$, and $S=0$ otherwise. If $c\ne a$, substituting \eqref{eq:overlap} in $Q$ gives
\begin{equation}
 Q\le-\frac35D_a-\frac{43}{15}\tau+S+\frac{20}{3}e_q
 \le S+\frac{20}{3}e_q.
 \label{eq:overlap_residual}
\end{equation}
The reference account satisfies
\begin{equation}
 Q\le S+11z_q(e_a+e_q)+112M\E_r(\log B_q'-\log B_q).
 \label{eq:reference_certificate}
\end{equation}
If $a=q$, then $D_a=0$ and $Q\le0$, while the reference account is unchanged. If $a\ne q$ and $B_q<1$, the target $q$ is chosen with probability $1/M$, and \eqref{eq:log_payment} gives
\[
 112M\E_r(\log B_q'-\log B_q)
 \ge\frac{28}{5}\gamma_{aq}-\frac{14}{15}\tau.
\]
Since $\gamma_{aq}\ge D_a$, subtracting this from $Q$ leaves at most
$-D_a/5-49\tau/15-2g/3\le0$.

It remains to consider $a\ne q$ and $B_q\ge1$. If $x_q>0$, the score envelopes give $e_a+e_q\ge x_q+D_a$, and by \eqref{eq:log_bounds} the negative log increment costs at most $112x_q/11$. Hence
\[
 Q+\frac{112}{11}x_q
 \le\frac{112}{11}(e_a+e_q)
 +\left(\frac{27}{5}-\frac{112}{11}\right)D_a
 -\frac{21}{5}\tau-\frac23g
 \le11(e_a+e_q).
\]
If $x_q=0$, the log increment vanishes; the reference is then an overlap candidate, so $c\ne a$, and \eqref{eq:overlap_residual} proves the claim.

\subsection{From errors to advancing streams}
Combining \eqref{eq:arithmetic_cost} and \eqref{eq:reference_certificate} with the potential \eqref{eq:capital_potential} gives
\[
 \E_r[\mathcal R_r+V_{r+1}-V_r]
 \le\mathcal E+S+11z_q(e_a+e_q).
\]
The anchor and the challenger are sampled in every packet, so $h_a=h_c=1$, and every other arm has $h_i\ge z_i/M$. Set the term $11z_qe_q$ aside; it is charged to the reference separately, even when the reference is the challenger. The remaining coefficient of the anchor is at most
\[
 \frac65+\frac65+\frac{22}{3}+11=\frac{311}{15}<21.
\]
For a distinct challenger it is at most $6/5+2/3+6/(5M)<21$, and for every other arm it is $6z_i/(5M)\le21h_i$. Finally, $11z_qe_q\le11M h_qe_q$. This proves \eqref{eq:master}.

\section{Signed stability and sampling}
\label{app:allprobe}
\subsection{The local mirror inequality}
\label{app:ap_stability}
Let $\Psi(p)=-2\sum_i\sqrt{p_i}$, let $p$ have positive coordinates, and let $p^+$ minimize $\eta\langle g,u\rangle+D_\Psi(u,p)$ over $u\in\Delta_m$. Every signed vector $g$ with $\eta\sqrt{p_i}g_i\ge-1/2$ satisfies
\begin{equation}
 \langle p-u,g\rangle\le
 \frac{D_\Psi(u,p)-D_\Psi(u,p^+)}\eta
 +2\eta\sum_i p_i^{3/2}g_i^2,\qquad u\in\Delta_m.
 \label{eq:signed_mirror}
\end{equation}
By the three-point inequality, it suffices to bound
$\eta\langle g,p-p^+\rangle-D_\Psi(p^+,p)$.
The scalar function $f(x)=-2\sqrt x$ on $x\ge0$ has conjugate $f^*(\theta)=-1/\theta$ for $\theta<0$. For $p_i>0$ and $s_i=\eta\sqrt{p_i}g_i>-1$, direct maximization gives
\begin{equation}
 \sup_{z\ge0}\{\eta g_i(p_i-z)-D_f(z,p_i)\}
 =\frac{\eta^2p_i^{3/2}g_i^2}{1+s_i},\qquad
 z^*=\frac{p_i}{(1+s_i)^2}.
 \label{eq:exact_conjugate}
\end{equation}
Relaxing the simplex to the nonnegative orthant and summing \eqref{eq:exact_conjugate} over $i$ proves \eqref{eq:signed_mirror}. Boundary comparators such as $e_j$ are allowed, because $D_\Psi(e_j,p)$ is finite and the gradient is evaluated only at positive $p$.

For the contrast estimator, $q_i\ge\sqrt{p_i}/(2Z)$ and $|\ell_i-\ell_A|\le1$ give $|s_i|\le2\eta Z\le1/8$, so the signed condition holds in every round.

\subsection{Well-defined update and sampling roles}
Strict convexity and the infinite negative derivative of $\Psi$ at the boundary place the minimizer in the interior of the simplex, and Lagrange optimality gives
\[
 p_i^+=\bigl(p_i^{-1/2}+\eta g_i+\lambda\bigr)^{-2},
 \qquad \sum_i\bigl(p_i^{-1/2}+\eta g_i+\lambda\bigr)^{-2}=1.
\]
On $\lambda> -\min_i(p_i^{-1/2}+\eta g_i)$, the sum is continuous and strictly decreasing from infinity to zero, so the normalizing root is unique.

The mixture in \eqref{eq:ap_sampler} gives the uniform bound $V=O(\sqrt m\,\gamma)$ behind the $O(m)$ regret guarantee, and neither component achieves this bound alone. For $q_i=\sqrt{p_i}/Z$ and losses $\ell_1=0$, $\ell_i=1$ for $i>1$, a direct calculation gives $V/\gamma=2Z^2\sqrt{p_1}$; taking $p_1=p_2=1/4$ and spreading the remaining mass uniformly gives $V/\gamma=Z^2=\Theta(m)$. For $q=p$, two arms, losses $(0,1)$, and $p=(1-\varepsilon,\varepsilon)$,
\[
 V/\gamma=\frac1{\sqrt{1-\varepsilon}}+\frac1{\sqrt\varepsilon},
\]
which is unbounded as $\varepsilon\downarrow0$.

\subsection{One pathwise certificate for both winner-feedback updates}
\label{app:shared_ftrl}
For $\mathcal E_m(p)=2(\sqrt m-\sum_i\sqrt{p_i})\ge0$, let
\[
 p_t=\arg\min_{p\in\Delta_m}\{\langle L_{t-1},p\rangle+\beta_t\mathcal E_m(p)\},
 \qquad L_t=L_{t-1}+g_t,
\]
where $L_0=0$ and $\beta_t>0$ is nondecreasing. The coordinates are
$p_{t,i}=(L_{t-1,i}/\beta_t+\lambda_t)^{-2}$, where $\lambda_t>-\min_i(L_{t-1,i}/\beta_t)$ is the unique normalizing scalar. If $\sqrt{p_{t,i}}g_{t,i}/\beta_t>-1$ for all $t$ and $i$, then, pathwise for every $u\in\Delta_m$,
\begin{equation}
 \sum_{t\le T}\langle p_t-u,g_t\rangle
 \le\beta_T\mathcal E_m(u)+
 \sum_{t\le T,i}\frac{p_{t,i}^{3/2}g_{t,i}^2}{\beta_t+\sqrt{p_{t,i}}g_{t,i}}.
 \label{eq:shared_ftrl}
\end{equation}
To prove it, put $F_t(L)=\min_p\{\langle L,p\rangle+\beta_t\mathcal E_m(p)\}$. The scalar identity \eqref{eq:exact_conjugate}, with $\eta=1/\beta_t$, bounds
$\langle p_t,g_t\rangle-F_t(L_t)+F_t(L_{t-1})$ by the $t$th term of the stability sum. Since $F_{t+1}(L)\ge F_t(L)$ and $F_1(0)=0$, the potential increments telescope to at most $F_T(L_T)\le\langle L_T,u\rangle+\beta_T\mathcal E_m(u)$.

At constant $\beta$, this iterate coincides with the mirror update with step $1/\beta$, because the optimality conditions of the mirror steps accumulate to $L_{t-1,i}-\beta/\sqrt{p_{t,i}}+\nu_t=0$. The label policy uses nonpositive update vectors and constant $\beta$, while the numerical policy uses nonnegative update vectors and $\beta_v=\sqrt v$.

\section{Budget lifting for fixed-array contrasts}
\label{app:contrast_budgets}
Since the contrast interface compares disjoint blocks, we use disjoint groups and, near full coverage, an activation rule funded by observed same-round gains.

\subsection{Sparse-budget ceiling effects}
\label{app:budget_constants}
Assume $3k<2n$, set $s=\lfloor k/2\rfloor$, and partition the arms into $m=\lceil n/s\rceil\ge4$ nonempty groups of size at most $s$. Every pair of groups is feasible, and the contrast of their maxima is observed directly, so Theorem~\ref{thm:ap_pair} applies to the group maxima as virtual arms. For $m=4$, $k/d<2$ gives regret below $256d/k$. For $m\ge5$, we have $n\ge(m-1)s+1$, $k\le3s$, and $d\ge(m-3)s$, so
\[
 \frac{32mk}{d}\le\frac{96m}{m-3}\le240<256.
\]
Since each group maximum dominates every member of its group, the same bound holds against every original arm.

\subsection{The contrast gate on a fixed array}
\label{app:contrast_gate}
For $3k\ge2n$, choose a uniformly random baseline set $G_0$ of size $n-2d$ and split its complement into two $d$-sets $G_1,G_2$. Every pair of groups is feasible. Write $Y_{g,t}=\max_{i\in G_g}x_{i,t}$ and condition on the partition. In each gate round, draw $g\in\{1,2\}$ uniformly, probe the ordered blocks $(G_0,G_g)$, and start a fresh three-arm contrast policy in the next round with probability $(Y_{g,t}-Y_{0,t})_+/L$, where $L=41>32(3-\sqrt3)$.

Let $\chi_t$ indicate that the gate is running, and let $\gamma_t=\frac12\sum_g(Y_{g,t}-Y_{0,t})_+$. At most one trigger occurs, so
\begin{equation}
 \E\sum_{t\le T}\chi_t\gamma_t
 =L\Prb(\text{trigger by }T)\le L.
 \label{eq:observable_payment_detail}
\end{equation}
Conditional on its start time and the history, the continuation faces a fixed suffix of the array and uses fresh randomness, so its regret against any fixed comparator is at most $L$. For a comparator $j$, put $c_t=x_{j,t}-Y_{0,t}$, and let $s_T$ be the probability that the continuation starts by $T$. Since $Ls_T\le\E\sum_t\chi_t\gamma_t$,
\[
 R_T(j)\le\E\sum_{t\le T}\chi_t(c_t-\gamma_t)+Ls_T
 \le\E\sum_{t\le T}\chi_tc_t.
\]
The right-hand side is nonpositive for $j\in G_0$; otherwise, $(c_t)_+\le2\gamma_t$ and \eqref{eq:observable_payment_detail} bound it by $2L$. A fixed arm misses $G_0$ with probability $2d/n$, so
\[
 R_T^{\mathrm{seq,dense}}\le4L\frac dn=164\frac dn.
\]

\begin{proof}\textbf{of the sequence part of Theorem~\ref{thm:main}.}
The sparse and dense policies above give the long-horizon bound $256d/k$. For $T<n/k$ in the sparse regime, $R_T\le T<3\delta T$. In the dense regime, $n/k\le3/2$, so only $T=1$ remains, and the initial gate action includes each arm with probability $k/n$, which gives $R_1\le\delta$. The coverage lower bound and the simulation of \CT by \AP complete \eqref{eq:sequence_law}.
\end{proof}

\section{The winner-feedback budget law}
\label{app:dependence}
We prove Theorem~\ref{thm:dependence} for all revelation rates at once. Throughout, $d=n-k$, $\delta=d/n$, and the comparator is fixed before the policy's randomization.

\subsection{Independent proposals and uniform completion}
\label{app:bp_labels}
Use the state of Appendix~\ref{app:shared_ftrl} with $\beta=2\sqrt n$, $q_i=(p_i+1/n)/2$, and $r=\min\{k,\lfloor n/2\rfloor\}$. Draw $r$ proposals independently from $q$, retain their set $D$ of distinct labels, and complete it with uniformly random labels to a $k$-set. Observe the winner and set
\begin{equation}
 g_i=-\ind\{J=i\}/q_i.
 \label{eq:bp_winner_update}
\end{equation}
\paragraph{The sampling inequality.}
Condition on the history and the current reward vector. Fix $j$, and let $L=\{i:x_i<x_j\}$ and $\ell=|L|$. Positive shortfall requires the completed set to lie in $L$, which is impossible if $\ell<k$. Otherwise, fix $r-1$ proposals with $m\le r-1$ distinct labels in $L$, and compare a final proposal in $L$ with a final proposal equal to $j$. In the latter case, completing the remaining slots inside $L$ makes $j$ the unique winner. The ratio of the all-low to the unique-winner completion probability is
\[
 \frac{\ell-k+1}{n-m}\quad\text{if the low proposal repeats a label},
 \qquad
 \frac{\ell-k+1}{\ell-m}\quad\text{otherwise}.
\]
Both ratios are at most $d/(n-r)$; for the second, use
$(\ell-k+1)/(\ell-m)=1-(k-m-1)/(\ell-m)$ and $\ell\le n-1$.
Now sum over the low final proposal and the other proposals. The $r$ events that specify the unique position of a proposal equal to $j$, with all other proposals in $L$, are disjoint, and hence
\begin{equation}
 \E(x_j-W)_+\le\Prb(S\subseteq L)
 \le\frac{d}{r(n-r)}\frac{\Prb(J=j)}{q_j}.
 \label{eq:bp_sampler}
\end{equation}
A proposed label is never omitted, and any other label is omitted with probability $d/(n-|D|)$. Hence
\begin{equation}
 \Prb(j\notin S)\le\frac d{n-r}\le2\delta.
 \label{eq:bp_protect}
\end{equation}
Because $L$ is a strict lower set, the argument is valid under any fixed tie order.

\paragraph{The update.}
Write $\pi_i=\Prb(J=i)$ and $a_i=\pi_i/q_i$. Since $q_i\ge p_i/2$ and $q_i\ge\sqrt{p_i/n}$,
\[
 \langle p,a\rangle\le2,\qquad
 \E\sum_i p_i^{3/2}g_i^2\le2\sqrt n,\qquad
 \frac{\sqrt{p_i}|g_i|}{\beta}\le\frac12.
\]
The stability term in \eqref{eq:shared_ftrl} therefore has expectation at most two per round. Since $\E g=-a$, comparing with $u=e_j$ gives
\[
 \E\sum_{t\le T}a_{j,t}\le4(n-\sqrt n)+4T.
\]
Combining \eqref{eq:bp_sampler} and \eqref{eq:bp_protect} with $r\ge k/2$ and $n-r\ge n/2$ proves
\begin{equation}
 \E\sum_{t\le T}(x_{j,t}-W_t)_+
 \le\min\left\{\frac d{n-r}T,
 \frac{4d(n-\sqrt n+T)}{r(n-r)}\right\}
 \le16\delta\min\{T,(n+T)/k\}.
 \label{eq:bp_ordinal}
\end{equation}
The first action is a uniform $k$-set by exchangeability. The same pathwise argument applies to a random initial prefix of length $A\le T$, provided that each round's membership in the prefix is decided before its action and observation:
\begin{equation}
 \E\sum_{t\le A}(x_{j,t}-W_t)_+
 \le16\delta\min\{\E A,(n+\E A)/k\}.
 \label{eq:bp_predictable}
\end{equation}
The comparator penalty is unchanged, and every per-round term is counted through $\E A$.

\subsection{Numerical groups and the observation clock}
\label{app:bp_upper}
Set $a=\min\{k,d\}$, $c=k-a$, and $m=\lceil n/k\rceil$. Choose a uniformly random core $G_0$ of size $c$, and partition the remaining $d+a$ labels into $m$ nonempty groups of size at most $a$. Such a partition exists because $(m-1)a\ge d$ and $m\le d+a$, and every action $G_0\cup G_g$ is feasible. For a fixed comparator $j$,
\[
 \Prb(j\notin G_0)=\frac{d+a}{n}\le2\delta,
 \qquad\text{and}\qquad m\le2n/k.
\]
Conditional on the core, write $v_{g,t}$ for the reward of the $g$th augmented group and $\ell_{g,t}=1-v_{g,t}$; a fixed group containing $j$ dominates $x_{j,t}$.

In each physical round of the numerical phase, the symmetric importance-weighted square-root learner~\citep{zimmert2021tsallis} draws $I_t\sim p_t$ afresh and updates only on revealed rounds:
\[
 g_{t,g}=\frac{(1-Y_t)\ind\{I_t=g\}}{p_{t,g}},
 \qquad \beta_v=\sqrt v,
\]
where $v$ indexes the current update and $Y_t$ is the reported bit, or a Bernoulli rounding of a revealed value~\citep[Section~3]{hanna2022bits}. Conditional on the pre-round history and $C_t=1$,
\[
 \E_tg_{t,g}=\ell_{g,t},\qquad
 \E_t\sum_g p_{t,g}^{3/2}g_{t,g}^2
 =\sum_g\sqrt{p_{t,g}}\ell_{g,t}\le\sqrt m.
\]
For $\rho>0$, let $\chi_t$ indicate that the numerical phase is active, $H_T=\sum_t\chi_tC_t$, and $v_t=1+\sum_{s<t}\chi_sC_s$. For a fixed group comparator $g^*$, independence of the current revelation coin gives
\begin{equation}
 \rho\E\sum_t\chi_t(v_{g^*,t}-W_t)
 =\E\sum_t\chi_tC_t\langle p_t-e_{g^*},g_t\rangle.
 \label{eq:observation_clock}
\end{equation}
This identity accounts for all unrevealed rounds, including a final stretch without revelation. Apply \eqref{eq:shared_ftrl} to the $H_T$ actual updates. Its comparator penalty is at most $2(\sqrt m-1)\E\sqrt{H_T}$, and, since the update index is predictable, its expected stability term is at most
\[
 \sqrt m\,\E\sum_{t\le T}\frac{\chi_tC_t}{\sqrt{v_t}}
 =\sqrt m\,\E\sum_{v=1}^{H_T}v^{-1/2}\le2\sqrt m\,\E\sqrt{H_T}.
\]
An empty update sequence contributes zero. Conditional on the core, the expected regret against a fixed group is therefore at most $4\sqrt m\,\E\sqrt{H_T}/\rho$. A comparator inside the core has nonpositive regret, and one outside it is dominated by its group. Since the phase switch uses only the revelation clock, $H_T$ is independent of the freshly chosen core. Averaging over the core yields
\begin{equation}
 R_T^{\mathrm{num}}(j)
 \le\frac{12\delta}{\rho}\sqrt{\frac nk}\,\E\sqrt{H_T}.
 \label{eq:bp_numerical_clock}
\end{equation}

\paragraph{One count-based switch.}
Let $M=nk$ and $N_T=\sum_{t\le T}C_t$. Use the label learner while $N_{t-1}<M$, that is, through the round of the $M$th revelation, and then start the numerical learner with fresh randomization. The label duration $A_T$ satisfies
\[
 \rho\E A_T=\E\min\{N_T,M\}\le M,\qquad A_T\le T,
 \qquad H_T=(N_T-M)_+.
\]
Put $A=\min\{T,n/k\}$ and $B_\rho=\min\{T/k,\sqrt{nT/(k\rho)}\}$. Since
\[
 \min\{T/k,n/\rho\}\le B_\rho,
\]
the predictable label bound \eqref{eq:bp_predictable} gives at most $16\delta(A+B_\rho)$ for the label phase. The inequality $\sqrt{(z-M)_+}\le\min\{\sqrt z,z/(2\sqrt M)\}$ and $\E N_T=\rho T$ bound \eqref{eq:bp_numerical_clock} by $12\delta B_\rho$. As $A,B_\rho\le\mathcal R_{n,k,\rho}(T)/\delta$, the total is at most $44\mathcal R_{n,k,\rho}(T)$. At $\rho=0$, the algorithm stays in the label phase. The proof applies to every fixed array, and averaging over joint i.i.d.\ arrays gives the corresponding guarantee against the best mean.

\subsection{An affine-density lower bound}
\label{app:bp_lower}
For $u\in[0,1/2]$, let $P_u$ have density $1+u\,\mathrm{sgn}(2x-1)$ on $(0,1)$. Its CDF, quantile function, and mean are
\[
 G_u(x)=x-u\min\{x,1-x\},\quad
 Q_u(v)=\min\left\{\frac v{1-u},\frac{v+u}{1+u}\right\},
 \quad \E_{P_u}X=\frac12+\frac u4.
\]
The quantile is strictly increasing in $u$ for $v\in(0,1)$, and mixing these laws mixes their parameters. Use the joint i.i.d.\ family \eqref{eq:bp_threshold}. The hidden arm has mean $1/2+\tau(1-\varepsilon/2)/4$, and every other arm has mean $1/2+\tau/8$.

For an $s$-set, the largest selected uniform $R$ has mean $s/(s+1)$. If $q$ is selected, the conditional numerical parameter is $\tau(1-\varepsilon+\varepsilon R)$. If not, the law at $R=z$ is
\[
 zP_{\tau(1-\varepsilon+\varepsilon z)}+(1-z)P_{\tau\varepsilon z}
 =P_{\tau z}.
\]
Integrating gives the laws $P_{h_s}$ and $P_{m_s}$ of the maximum when $q$ is included and omitted, respectively, where $h_s=\tau(1-\varepsilon/(s+1))$ and $m_s=\tau s/(s+1)$. In both cases, $J$ is uniform and independent of $W$: the argmax of the selected uniforms is independent of their maximum, the winner always clears the threshold $U_q$ when $q$ is selected, and otherwise $U_q$ is independent of the selected uniforms. The action cost is exactly
\begin{equation}
 r_q(S)=\frac{\tau}{4(s+1)}
 \bigl(\ind\{q\notin S\}+\varepsilon\ind\{q\in S\}\bigr)
 -\frac{\tau\varepsilon}{8}
 \ge\frac{\tau}{4(s+1)}\ind\{q\notin S\}-\frac{\tau\varepsilon}{8}.
 \label{eq:bp_action_cost}
\end{equation}

A value from $P_u$ is $(B+Z)/2$, where $B\sim\Bern((1+u)/2)$ and $Z$ is an independent uniform. The bit is recovered as $\ind\{W>1/2\}$, and the position within each half carries no information about $u$. Consequently,
\begin{equation}
 \operatorname{KL}(P_u\Vert P_v)
 =\operatorname{kl}\left(\frac{1+u}{2},\frac{1+v}{2}\right)
 \le\frac43(u-v)^2,
 \label{eq:bp_kernel_kl}
\end{equation}
where we used $\operatorname{kl}(p,z)\le(p-z)^2/[z(1-z)]$ and $z\in[1/2,3/4]$. The bound includes singleton actions, and erasure multiplies the divergence by $\rho$.

\paragraph{Inspect before observing.}
For each size $s$, the reference experiment $Q$ uses a numerical law with parameter
$\bar u_s=(s/n)h_s+(1-s/n)m_s$, an independent uniform winner, and the same revelation coin. Under $Q$, the uniform hidden label is independent of the complete policy transcript, and averaging \eqref{eq:bp_kernel_kl} over this label gives
\[
 \frac1n\sum_q\operatorname{KL}(Q_s^\rho\Vert P_{q,S}^\rho)
 \le\frac{4\rho\tau^2}{3}\frac{s(n-s)}{n^2(s+1)^2}
 \le\frac{4\rho\tau^2}{3n}\frac{n-s}{n(s+1)}.
\]
Put $a_t=1/(|S_t|+1)$ and $b=T/(k+1)$. Independently draw $Z_0\sim\mathrm{Uniform}(0,b)$, unavailable to the policy, and define
\[
 \sigma=\min\{t:\textstyle\sum_{v\le t}a_v\ge Z_0\}\le T,
 \qquad E=\{q\notin S_\sigma\}.
\]
The recorded experiment contains $q$, $Z_0$, the actions through $S_\sigma$, and the observations through round $\sigma-1$ only. Denote its laws by $\bar Q$ and $\bar P$, and let $p_0=\bar Q(E)$ and $p_1=\bar P(E)$. Generate each full trajectory without using $Z_0$, and write $S(u)=S_t$ on the $t$th budget interval. Averaging over the independent inspection point gives
\begin{equation}
 p_0=\frac1b\E_Q\int_0^b\frac{n-|S(u)|}{n}\,du\ge\delta,
 \qquad
 \E_{\mathrm{true}}\sum_{t\le T}a_t\ind\{q\notin S_t\}\ge bp_1.
 \label{eq:bp_reference_cost}
\end{equation}
Only feedback before the inspection enters the adaptive chain rule for KL, and the corresponding budget intervals lie entirely inside $(0,b)$, so
\begin{equation}
 \operatorname{KL}(\bar Q\Vert\bar P)
 \le\frac{4\rho\tau^2}{3n}\E_Q\sum_{t<\sigma}a_t\frac{n-|S_t|}{n}
 \le\frac{4\rho\tau^2b}{3n}p_0.
 \label{eq:bp_stopped_kl}
\end{equation}
Jensen's inequality applied to the likelihood ratio yields the relative event inequality
\begin{equation}
 (\log2)p_0-\operatorname{KL}(\bar Q\Vert\bar P)
 \le\log\E_{\bar P}2^{\ind_E}=\log(1+p_1)\le p_1.
 \label{eq:bp_relative}
\end{equation}
Choose $\tau=\frac12\min\{1,\sqrt{n/(\rho b)}\}$, with $\tau=1/2$ at zero rate. Then \eqref{eq:bp_stopped_kl} and \eqref{eq:bp_relative} give $p_1\ge p_0/3\ge\delta/3$, and applying the action cost over all $T$ physical rounds gives, uniformly over policies,
\[
 \frac1n\sum_qR_T^\pi(D_{q,\varepsilon})
 \ge\frac{\tau\delta b}{12}-\frac{\tau\varepsilon T}{8}.
\]
Taking the supremum over these strict-rank environments as $\varepsilon\downarrow0$ and using $b\ge2T/(3k)$ gives
\begin{equation}
 \cV_{n,T}^{\Fbk_\rho,\mathsf M,(k)}
 \ge\frac\delta{24}\min\{b,\sqrt{nb/\rho}\}
 \ge\frac\delta{36}\min\{T/k,\sqrt{nT/(k\rho)}\}.
 \label{eq:bp_noise_lower}
\end{equation}
For fixed arrays, the same lower bound follows by averaging over joint i.i.d.\ arrays and using $\E\max_j\sum_tX_{j,t}\ge T\mu_q$.

\subsection{Combining the costs}
\label{app:bp_matching}
For $A=\min\{T,n/k\}$, $B_\rho=\min\{T/k,\sqrt{nT/(k\rho)}\}$, and $F=\mathcal R_{n,k,\rho}(T)/\delta$,
\[
 A,B_\rho\le F\le A+B_\rho\le2F.
\]
For the middle inequality, if $T\le n/k$, then $A=T$; otherwise, $A=n/k$, and $A+B_\rho$ is either $(n+T)/k$ or at least the square-root term. The coverage and learning lower bounds therefore give
\[
 \cV\ge\max\{\delta A/2,\delta B_\rho/36\}
 \ge\delta(A+B_\rho)/72\ge\mathcal R_{n,k,\rho}(T)/72.
\]
Together with the upper bound $44\mathcal R_{n,k,\rho}(T)$ and the simulation of $\Bit_\rho$ by $\Fbk_\rho$, this proves Theorem~\ref{thm:dependence}. At $\rho=0$, it proves \eqref{eq:bp_labels_law}, and the same algebra proves \eqref{eq:bp_cost_decomposition}. The stronger label-only learning lower bound $\delta T/(12k)$ follows directly from \eqref{eq:bp_action_cost} with $\tau=1/2$ and $\varepsilon\downarrow0$.

\section{Numerical revelation and ordinal ambiguity under independence}
\label{app:numerical}\label{sec:numerical}
Throughout this appendix, the arms are independent and the reward vectors are i.i.d.

\subsection{Maximum-only feedback with erasures}
\label{app:maxonly_erasure}
\paragraph{The independent clock.}\label{the-independent-clock-reduction}
For $\rho>0$, the hold-until-revelation reduction~\citep[Algorithm~2]{dai2024probabilistic} repeats each action of the product policy of Theorem~\ref{thm:main}, including all three singleton slots of each packet, until its value is revealed. It passes only the revealed value to that policy and ignores all labels and waiting times. Generate the resulting virtual reward experiment independently of the geometric waiting times. Then $N_T\sim\operatorname{Binomial}(T,\rho)$ is independent of the virtual experiment, and since the current coin is fresh, we obtain the signed identity
\begin{equation}
 \rho R_T^{\mathrm{hold}}
 =\E\sum_{t\le T}C_t(\mu^\star-W_t)
 =\sum_{s=0}^T\Prb(N_T=s)R_s^\pi,
 \qquad R_0^\pi=0,
 \label{eq:clock}
\end{equation}
which also accounts for an unfinished final waiting interval. Concavity of $\Phi_{n,k}$ then yields
\begin{equation}
 R_T^{\mathrm{hold}}\le\frac{46000}{\rho}\E\Phi_{n,k}(N_T)
 \le46000\delta\min\{T,n/(k\rho)\}.
 \label{eq:held_upper}
\end{equation}

\begin{proof}\textbf{of Corollary~\ref{cor:maxonly_erasure}.}
The upper bound is \eqref{eq:held_upper}; neither the base policy nor the holding rule uses $T$ or $\rho$.

For the lower bound, let $x=n/k>1$ and hide a deterministic reward-one arm at a uniform label, with all other rewards zero. Fix the policy seed and the revelation coins, and follow the path on which every revealed maximum is zero. If the hidden label lies outside both the earlier revealed menus and the current menu, this path agrees with the true transcript and the current regret is one. Since this union contains at most $k(N_{t-1}+1)$ labels, where $N_{t-1}$ counts earlier revelations, the prior-average regret is at least
\begin{equation}
 \sum_{t=1}^T\E\left(1-\frac{N_{t-1}+1}{x}\right)_+.
 \label{eq:erased_search_lower}
\end{equation}
If $x\ge2$, Jensen's inequality bounds the summand below by $(\delta-\rho(t-1)/x)_+$, which is at least $\delta/2$ for the first $\min\{T,\lfloor x/(4\rho)\rfloor+1\}$ terms; this gives $\delta\min\{T,x/\rho\}/8$.

If $1<x<2$, retain only the event $N_{t-1}=0$. Then \eqref{eq:erased_search_lower} is at least
\[
 \delta\frac{1-(1-\rho)^T}{\rho}
 \ge(1-e^{-1})\delta\min\{T,1/\rho\}
 \ge\frac{1-e^{-1}}2\delta\min\{T,x/\rho\},
\]
which exceeds the stated bound. At $\rho=0$, \eqref{eq:erased_search_lower} equals $\delta T$, and uniformly random $k$-sets attain this value because they include a mean-optimal arm with probability $k/n$ in each round.
\end{proof}

\subsection{Winner labels with intermittent values}
Write $V^{\mathrm{val}}_{n,k,\rho}(T)=\cV_{n,T}^{\Fbk_\rho,(k)}$, $V^{\mathrm{lab}}_{n,k}(T)=V^{\mathrm{val}}_{n,k,0}(T)$, and
\[
 U_{n,k,\rho}(T)=\delta\min\{T,(n+T)/k,n/(k\rho)\},\quad
 \ell_x(T)=\min\{k,\log(e+T/x)\}.
\]
Set $h=1$ for $k\le n/2$, and $h=\log(e/\delta)$ for $k>n/2$.

\begin{theorem}[Finite-time cost of numerical revelation]
\phantomsection\label{thm:numerical}
For every $n>k\ge2$, $T\ge1$, and $\rho\in[0,1]$,
\begin{equation}
 \frac1{65536}\max\left\{\Phi_{n,k}(T),
 \delta\min\left\{\frac{T}{hk\ell_n(T)},\frac{n}{k\rho}\right\}\right\}
 \le V^{\mathrm{val}}_{n,k,\rho}(T)\le46032U_{n,k,\rho}(T).
 \label{eq:feedback_law}
\end{equation}
One anytime policy, knowing neither $T$ nor $\rho$, attains the upper bound. For labels alone, the lower bound sharpens to
\begin{equation}
 V^{\mathrm{lab}}_{n,k}(T)\ge
 \frac1{4096}\max\left\{\Phi_{n,k}(T),\frac{\delta T}{k\ell_d(T)}\right\}.
 \label{eq:label_soft_lower}
\end{equation}
\end{theorem}
The label-only bounds \eqref{eq:label_soft_lower} and \eqref{eq:bp_label_upper} differ by a factor $O(\ell_d(T))$, which grows logarithmically in $T$ until it saturates at $k$. For every positive revelation rate, the horizon-uniform cost is determined up to constants:
\begin{corollary}[Horizon-uniform product cost]
\phantomsection\label{cor:uniform_product}
For $\rho>0$,
\begin{equation}
 \inf_{\pi\ \mathrm{anytime}}\sup_{T\ge1}\sup_{D\ \mathrm{product}}R_T^\pi(D)
 \asymp\frac{n-k}{k\rho},
 \label{eq:uniform_product}
\end{equation}
with universal constants. Moreover, the finite-time value has this order whenever
\begin{equation}
 T\ge\frac{8hn}{\rho}\min\{k,\log(e+h/\rho)\}.
 \label{eq:saturation_time}
\end{equation}
\end{corollary}
\subsection{One count-based switch}
To attain the envelope $U_{n,k,\rho}$ in \eqref{eq:feedback_law}, run the label policy of Appendix~\ref{app:bp_labels} through the $n$th numerical revelation, and start a fresh instance of the held product policy in the following round. The label duration $A_T$ satisfies $\rho\E A_T\le n$ and $A_T\le T$, so \eqref{eq:bp_predictable} bounds the label phase by $32U_{n,k,\rho}(T)$. Conditional on its start time and the history, the product continuation costs at most $46000\delta n/(k\rho)$ by \eqref{eq:held_upper}, and it starts by $T$ with probability at most
\[
 \Prb(N_T\ge n)\le\min\{1,\rho T/n\}.
\]
Its total expected cost is therefore at most $46000\delta\min\{T/k,n/(k\rho)\}\le46000U_{n,k,\rho}(T)$, which proves the upper bound in \eqref{eq:feedback_law}. At zero rate, the policy remains label-only.

\subsection{Exact hiding has a sharp endpoint limit}
\label{app:rank_moments}
Consider a hidden candidate among continuous i.i.d.\ ordinary arms with CDF $F_0$, and let $Z=F_0(X_q)$ be its rank. Its winning probability in an $s$-set is $\E Z^{s-1}$, so uniform labels at all sizes up to $k$ require the first $k-1$ moments of $Z$ to match those of a uniform variable.

\begin{lemma}[Exact rank matching]
\phantomsection\label{lem:rank_moments}
Among probability laws $\nu$ on $[0,1]$ with
\begin{equation}
 \int z^j\nu(dz)=\frac1{j+1}\quad(0\le j\le k-1),
 \label{eq:rank_moments}
\end{equation}
the largest endpoint mass is
\[
 p_k=\frac1{\lfloor(k+1)^2/4\rfloor}.
\]
A finitely supported law attains it.
\end{lemma}
\begin{proof}
Let $L_j(z)=P_j(2z-1)$ be the shifted Legendre polynomials~\citep[Chapter~18]{dlmf}. At odd $k=2m+1$, the nonnegative polynomial
\[
 F(z)=\left[\frac{\sum_{j=0}^m(2j+1)L_j(z)}{(m+1)^2}\right]^2
\]
has degree $2m$, $F(1)=1$, and integral $1/(m+1)^2$. At even $k=2m$, use
\[
 F(z)=z\left[\frac{L'_m(z)}{m(m+1)}\right]^2.
\]
It has degree $2m-1$, $F(1)=1$, and integral $1/[m(m+1)]$: integration by parts gives $\int(L'_m)^2=2m(m+1)$, and symmetry gives $\int z(L'_m)^2=m(m+1)$. In either case, $F$ is nonnegative, equals one at the endpoint, and has degree at most $k-1$, so $\nu(\{1\})\le\int F\,d\nu=\int_0^1F$.

For attainment, use the $(m+1)$-node Gauss--Lobatto rule when $k=2m$ and the right Gauss--Radau rule when $k=2m+1$. These rules have positive weights and are exact through degrees $2m-1$ and $2m$, respectively~\citep{joulak2009gautschi}. The Lobatto nodes are $0,1$ and the zeros of $L'_m$. The interior Radau nodes are the zeros of $K_m(z)=\sum_{j=0}^m(2j+1)L_j(z)$: for every polynomial $p$ of degree below $m$, the reproducing identity of $K_m$ gives $\int_0^1(1-z)K_m(z)p(z)\,dz=p(1)-(zp)(1)=0$. In each case, the corresponding $F$ vanishes at every node except $1$, so exactness applied to $F$ identifies the endpoint weight as $\int_0^1F=p_k$.
\end{proof}

For $0<\varepsilon\le1/64$, use ordinary rewards $D_0=\mathrm{Uniform}[1/2-\varepsilon,1/2+\varepsilon]$. At the hidden label $q$, draw a rank from the extremal law of Lemma~\ref{lem:rank_moments}, map ranks below one linearly onto this interval, and map rank one to reward one. By \eqref{eq:rank_moments}, every selected winner is uniform, independently of $q$, while the hidden mean is $1/2+p_k(1/2-\varepsilon)$. With $r_q(S)=\mu_q-\E_q\max_{i\in S}X_i$, every action satisfies
\begin{equation}
 r_q(S)\ge\frac{p_k}{2}\ind\{q\notin S\}-2\varepsilon.
 \label{eq:general_rank_cost}
\end{equation}
Thus labels alone incur linear regret at fixed $n$ and $k$. Approximate matching of the label kernels permits larger hidden masses at finite horizons.

\subsection{A rare spike with a small rank defect}
\label{app:soft_rank}
For $2\le r\le k$, set
\begin{equation}
 K_r(z)=\sum_{j=0}^{r-1}(2j+1)L_j(z),\qquad
 \nu_{r,p}(dz)=(1-pK_r(z))\,dz+p\delta_1(dz),
 \quad0<p\le\frac1{64r^2}.
 \label{eq:soft_rank}
\end{equation}
Since $|K_r|\le r^2$, $\int K_r=1$, and $\int K_r^2=r^2$, this is a probability law whose continuous part has a positive density, and its mean rank is $1/2$.

\begin{lemma}[Soft rank defect]
\phantomsection\label{lem:soft_defect}
For $D_{s,r}=1-\int_0^1z^{s-1}K_r(z)\,dz$,
\begin{equation}
 D_{s,r}=0\ (s\le r),\qquad
 D_{s,r}=\prod_{j=0}^{r-1}\frac{s-1-j}{s+j}
 \le e^{-r^2/s}\ (s>r).
 \label{eq:soft_defect}
\end{equation}
\end{lemma}
\begin{proof}
Rodrigues' formula and $j$ integrations by parts give
\[
 \int_0^1z^mL_j(z)\,dz
 =\frac{(m!)^2}{(m-j)!(m+j+1)!}\quad(m\ge j),
\]
and zero for $m<j$. Multiplication by $2j+1$ telescopes as $A_j-A_{j+1}$, where $A_j=(m!)^2/[(m-j)!(m+j)!]$, which proves the product formula. Summing $\log((x-u)/(x+u))\le-2u/x$ with $x=s-1/2$ and $u=j+1/2$ proves the exponential bound.
\end{proof}

Map this rank law to rewards as before, with ranks below one sent linearly onto the ordinary interval and the atom sent to reward one. Then
\begin{equation}
 \mu_q=\tfrac12+p(\tfrac12-\varepsilon),\qquad
 r_q(S)\ge\frac p3\ind\{q\notin S\}-2\varepsilon
 \quad(0<\varepsilon\le1/64).
 \label{eq:soft_cost}
\end{equation}
If $q\in S$ and $s=|S|\ge2$, the label probabilities are
\begin{equation}
 \Prb_q(J=q)=1/s+pD_{s,r},\qquad
 \Prb_q(J=j)=1/s-pD_{s,r}/(s-1)\quad(j\in S\setminus\{q\}).
 \label{eq:soft_label_kernel}
\end{equation}
If $q\notin S$, the label kernel is exactly uniform, and singleton labels are uninformative.

Let $Q$ be the environment in which all arms are ordinary. When $r^2\ge k$, the chi-square bound for categorical laws and $sp\le1/64$ give
\begin{equation}
 \operatorname{KL}(Q^S_{\mathrm{lab}}\Vert P^S_{q,\mathrm{lab}})
 \le4sp^2D_{s,r}^2\ind\{q\in S\}.
 \label{eq:soft_label_kl}
\end{equation}
For numerical observations, data processing from the full vector of selected rewards gives
\[
 \operatorname{KL}(D_0\Vert D_q)
 =\int_0^1-\log(1-pK_r(z))\,dz
 \le p+2p^2r^2\le2p.
\]
Consequently, the full feedback kernel satisfies
\begin{equation}
 \operatorname{KL}(Q^S\Vert P_q^S)
 \le\ind\{q\in S\}\bigl(2\rho p+4(1-\rho)sp^2D_{s,r}^2\bigr).
 \label{eq:soft_full_kl}
\end{equation}
Numerical identification contributes a term of first order in $p$, whereas the ordinal term is quadratic in $p$ and exponentially damped by rank matching.

\subsection{Finite-time lower bounds and saturation}
\label{app:product_finite_lower}
Fix a policy. For each parameter choice below, let $\bar Q$ and $\bar P$ be the joint laws of a uniform hidden label $q$ and the full $T$-round transcript under the all-ordinary reference and under the hidden family, respectively. Write $D^{(q)}$ for the reward-vector law of the hidden family and $\overline R_T=n^{-1}\sum_qR_T^\pi(D^{(q)})$. Both laws are augmented with the same independent inspection variables.

\paragraph{Labels.}
Take $L=\log(e+T/d)$, $r=\min\{k,\lceil\sqrt{kL}\rceil\}$, $p=1/(64r^2)$, and $\varepsilon=p\delta/100$. Under $Q$, the uniformly distributed label $q$ is independent of every adaptive action. The divergence $K_0=\operatorname{KL}(\bar Q\Vert\bar P)$ is zero when $r=k$; otherwise \eqref{eq:soft_label_kl} gives, with $x=T/d$,
\[
 K_0\le\frac{4p^2k^2T}{n}e^{-2r^2/k}
 \le\frac\delta{1024}\frac{x}{L^2(e+x)^2}
 \le\frac\delta{4096e}.
\]
Inspect an independent uniformly random physical time, and let $F$ be the event that its action omits $q$. Its reference probability $x_0$ is at least $\delta$. If its true probability $y$ is below $x_0$, binary data processing and
\[
 \operatorname{kl}(x_0,y)
 =\int_y^{x_0}\frac{x_0-v}{v(1-v)}\,dv
 \ge\frac{(x_0-y)^2}{2x_0}
\]
give $y\ge x_0-\sqrt{2x_0K_0}\ge\delta/2$. By \eqref{eq:soft_cost}, the prior-average regret is therefore at least $p\delta T/8$, which is at least $\delta T/[2048k\ell_d(T)]$ because $r^2\le4k\ell_d(T)$. Combining this with the coverage bound proves \eqref{eq:label_soft_lower}.

\paragraph{True values, $k\le n/2$.}
Use
\[
 L=\log(e+T/n),\quad r=\min\{k,\lceil\sqrt{kL}\rceil\},\quad
 p=\min\{1/(64r^2),n/(128\rho kT)\},\quad\varepsilon=p\delta/100.
\]
The same joint KL is at most
\[
 \frac Tn\left(2\rho pk+4p^2k^2e^{-2r^2/k}\right)\le\frac1{32},
\]
where the ordinal term vanishes when $r=k$. Pinsker's inequality bounds the total variation by $1/8$, and since $\delta\ge1/2$, the omission probability remains at least $\delta/2$. Hence
\[
 \overline R_T\ge\frac{p\delta T}{8}
 \ge\delta\min\left\{\frac{T}{2048k\ell_n(T)},\frac{n}{1024k\rho}\right\}.
\]

\paragraph{True values, $k>n/2$.}
Put $h=\log(e/\delta)$ and use
\[
 L=\log(e+T/k),\quad r=\min\{k,\lceil\sqrt{kL}\rceil\},\quad
 p=\min\{1/(64hr^2),1/(64\rho T)\},\quad\varepsilon=p\delta/100.
\]
Give the learner all selected rewards on numerical rounds, and redefine $\bar Q,\bar P$ for this stronger experiment. Let $Z_q=d\bar P_{\mathrm{ac}}/d\bar Q$ be the likelihood ratio of the absolutely continuous part of the true law; histories containing a revealed spike form its singular part. For an action containing $q$, the one-step factor of $Z_q$ is $1-pK_r(U_q)$ on a numerical round and the ratio from \eqref{eq:soft_label_kernel} on a label round; for every other action it is one.

For $|u|\le1/(64h)$, Taylor's theorem gives
\[
 (1+u)^{-h}\le1-hu+2h(h+1)u^2.
\]
On numerical rounds, $\E u=-p$ and $\E u^2=p^2r^2$, so the expected inverse factor is at most $1+2hp$. On label rounds, $\E u=0$ and $\E u^2\le2sp^2D_{s,r}^2$, so it is at most $1+8h^2kp^2e^{-2r^2/k}$. Iterating conditionally along the adaptive reference history yields
\begin{equation}
 \frac1h\log\E_{\bar Q} Z_q^{-h}
 \le T\{2\rho p+8hkp^2e^{-2r^2/k}\}\le\frac1{16}.
 \label{eq:inverse_likelihood}
\end{equation}
The ordinal term again vanishes when $r=k$. For the omission event $F$ at the inspected time, the singular part of the true law has nonnegative mass, and H\"older's inequality gives
\[
 \bar P(F)\ge\E_{\bar Q}Z_q\ind_F
 \ge\bar Q(F)^{1+1/h}(\E_{\bar Q}Z_q^{-h})^{-1/h}
 \ge\delta^{1+1/h}e^{-1/16}\ge\delta/4.
\]
Thus $\overline R_T\ge p\delta T/16$. Using $r^2\le8k\ell_n(T)$ and $n/k<2$, and combining with coverage, gives \eqref{eq:feedback_law} with the stated constant.

\begin{proof}\textbf{of Corollary~\ref{cor:uniform_product}.}
The uniform upper bound follows from \eqref{eq:held_upper}. For the finite-time statement, write $A=h/\rho$, $b=\min\{k,\log(e+A)\}$, and $T_0=8nAb$. If $b=k$, then $T_0\ge An\ell_n(T_0)$; otherwise,
\[
 \log(e+8Ab)\le b+\log(1+8b)\le4b,
\]
which gives the same inequality. Since $x/\min\{k,\log(e+x)\}$ is increasing, $T\ge An\ell_n(T)$ for all $T\ge T_0$, and the second lower term in \eqref{eq:feedback_law} then equals $\delta n/(k\rho)$. This also gives the lower bound in \eqref{eq:uniform_product} for every anytime policy.
\end{proof}

\begin{corollary}[Exponential startup for a uniform $k^{-2}$ slope]
\phantomsection\label{cor:exponential_startup}
\leavevmode\par
Suppose label-only policies satisfy $R_T\le S(n,k)+C\delta T/k^2$ on every product instance, for one fixed finite $C>0$. Then some $c_C,\alpha_C>0$ satisfy
\[
 S(2k,k)\ge c_C e^{\alpha_C k}/k
\]
for all sufficiently large $k$, even when the policies know the horizon.
\end{corollary}
\begin{proof}
Apply \eqref{eq:label_soft_lower} with $n=2k$ and $T=\lceil k e^{\alpha_Ck}\rceil$, where $\alpha_C=1/[16384(C+1)]$. For large $k$, $\ell_d(T)\le2\alpha_Ck$, so the lower bound is at least $2(C+1)\delta T/k^2$. Subtracting the assumed slope term proves the claim.
\end{proof}
Exact rank matching permits a hidden gap of order $k^{-2}$, but a uniform regret bound with that slope requires an exponential startup. Whether such a bound is achievable remains open. The soft family thus separates the finite-time cost of resolving nearly identical label kernels from the smaller gap compatible with exact rank matching.

\subsection{Proof of the Bernoulli-bit separation}
\label{app:bit_separation}
We prove Proposition~\ref{prop:bit_separation}. Fix $n=3,k=2$, $0<\alpha\le1/12$, $\varepsilon=\alpha/100$, and $D_0=\mathrm{Uniform}[a,b]$, with $a=1/2-\varepsilon$, $b=1/2+\varepsilon$. At a uniform hidden label $q$, use
\[
 D_q=(1-\alpha)D_0+\frac\alpha2\delta_a+\frac\alpha2\delta_1.
\]
The hidden mean advantage is $g=\alpha(1/4-\varepsilon/2)\ge\alpha/5$. In a pair containing $q$, the hidden arm wins with probability $(1-\alpha)/2+\alpha/2=1/2$, so every legal label kernel is independent of $q$. The supports give, for every action,
\[
 r_q(S)\ge g\ind\{q\notin S\}-2\varepsilon.
\]

Let $Q$ have three ordinary arms, and put $m_s=1/2+\varepsilon(s-1)/(s+1)$. Under $Q$, a revealed bit has mean $m_s$ conditional on the winning label. In the hidden pair $\{q,j\}$, the conditional means are instead
\begin{equation}
 v_q=(1-\alpha)m_2+\alpha,\qquad
 v_j=(1-\alpha)m_2+\alpha/2.
 \label{eq:joint_label_bit}
\end{equation}
These conditional means differ by $\alpha/2$: the label alone is uninformative, but its dependence on the bit carries information. The singleton $\{q\}$ has bit mean $1/2+g$. All these means lie in $[1/2,3/4]$ and differ from the reference by at most $\alpha/2$. Since the label kernels agree, the chi-square bound for Bernoulli laws gives the joint-kernel bound
\[
 \operatorname{KL}(Q^S\Vert P_q^S)\le2\rho\alpha^2.
\]
Pinsker's inequality bounds the total variation between the adaptive history laws by $\alpha\sqrt{\rho T}$. Choose
$\alpha=\frac1{12}\min\{1,(\rho T)^{-1/2}\}$, with $\alpha=1/12$ at zero rate. If $O_q$ counts the rounds that omit $q$, every path has $\sum_qO_q\ge T$, so
\[
 \frac13\sum_q\E_qO_q\ge T/3-T/12=T/4.
\]
The prior-average regret is at least
\[
 \frac{\alpha T}{20}-\frac{2\alpha T}{100}
 =\frac1{400}\min\{T,\sqrt{T/\rho}\}.
\]
The upper bound $18\min\{T,\sqrt{T/\rho}\}$ follows from Theorem~\ref{thm:dependence} with three arms and two probes. At $\rho=1$, Theorem~\ref{thm:main} gives bounded regret from exact values, and the exact rank family gives linear regret from labels alone. This proves the proposition.

\section{Action constraints and coverage}
\label{app:constraints}
We first characterize when bounded regret is possible for downward-closed action families that contain all singletons, and then extend the coverage law to per-class quotas.

\subsection{The pair-feasibility threshold}
\begin{proposition}[Feasible pairs]
\phantomsection\label{prop:feasible_pairs}
Let $\mathcal F\subseteq2^{[n]}$ be downward closed and contain every singleton. Under product rewards and \MO, a distribution-uniform, horizon-independent regret bound against the best fixed single arm exists if and only if every pair belongs to $\mathcal F$. The same equivalence holds for \CT on fixed arrays.
\end{proposition}
\begin{proof}
If every pair is feasible, use Theorem~\ref{thm:pair} with zero core for $n\ge3$, or Theorem~\ref{thm:ap_pair}, respectively; for $n\le2$, select every arm. Conversely, if $\{i,j\}\notin\mathcal F$, downward closure forbids any legal action from containing both arms. Put an ordinary hard two-arm stochastic instance at these positions and zeros elsewhere. Every action then selects and observes at most one nonzero arm, even under \AP, and earns at most that arm's reward. The two-arm bandit lower bound gives $\Omega(\sqrt T)$ regret~\citep[Exercise~15.4]{lattimore2020bandit}, and averaging over its reward arrays transfers the bound to fixed arrays.
\end{proof}
In particular, connectivity of the feasible-pair graph does not suffice. For a knapsack constraint $\sum_{i\in S}c_i\le B$ with $0\le c_i\le B$, the condition is that the two largest costs sum to at most $B$.

\subsection{A partition-quota coverage law}
\label{app:quotas}
Partition the labels into classes $C_c$ of sizes $n_c$, with integer quotas $1\le k_c\le n_c$, and require $|S\cap C_c|\le k_c$. Write $d_c=n_c-k_c$ for each class. Suppose each nonfull class has $k_c\ge2$, and put
\[
 \beta=\min_c k_c/n_c,\qquad
 \Phi_\beta(T)=(1-\beta)\min\{T,1/\beta\}.
\]
\begin{proposition}[Heterogeneous quotas]
\phantomsection\label{prop:quotas}
For $\beta<1$, the minimax regret is $\Theta(\Phi_\beta(T))$, with universal constants independent of the number of classes, both for product rewards with \MO and for fixed arrays with \CT. Each upper bound is attained by an anytime policy. For $\beta=1$, the value is zero.
\end{proposition}
\begin{proof}
Assume $\beta<1$. Place a deterministic hidden good arm in a class attaining $\beta$ and zeros elsewhere. The quota of this class gives the lower bound $\Phi_\beta(T)/2$ by \eqref{eq:det_coverage}, even under \AP.

\paragraph{Product rewards.}
For each nonfull class, set
\[
 s_c=\min\{d_c,\lfloor k_c/2\rfloor\},\quad
 m_c=2+\lceil d_c/s_c\rceil,\quad m=\max_c m_c.
\]
Choose a uniformly random core of size $k_c-2s_c$ within each such class, and distribute its remaining labels across $m$ global groups, with at most $s_c$ labels in each group. Empty class pieces are allowed, and a class attaining $m$ supplies a nonempty piece to every group. Put every full class in the core. The core together with any two global groups then respects every quota.

Conditional on the partition, the disjoint base maxima and the core maximum are independent, so Theorem~\ref{thm:pair} applies. An arm in a full class is always covered, and an arm in a nonfull class $c$ is outside the core with probability
\[
 p_c=(d_c+2s_c)/n_c\le3(1-\beta).
\]
If $m=3$, then $p_cm\le9(1-\beta)/\beta$. If $m\ge4$, choose a class $h$ attaining $m$. The integer calculation in Lemma~\ref{lem:core_geometry} gives $k_hm\le8d_h$, hence $m\le8(1-\beta)/\beta$, and the same bound on $p_cm$ follows. Applying the core argument of Section~\ref{sec:budgets} to a fixed optimal arm yields
\[
 R_T\le3(1-\beta)T,\qquad
 R_T<45558(1-\beta)/\beta.
\]
Thus the product upper bound is less than $46000\Phi_\beta(T)$.

\paragraph{Fixed arrays.}
If $\beta<2/3$, distribute each nonfull class across
\[
 m=\max_{c:k_c<n_c}\left\lceil\frac{n_c}{\lfloor k_c/2\rfloor}\right\rceil<4/\beta
\]
disjoint global groups, respecting the class capacity $\lfloor k_c/2\rfloor$ per group; a maximizing class makes every group nonempty. Full classes can be assigned to one group. Every pair of groups is feasible, so Theorem~\ref{thm:ap_pair} costs less than $32m<384(1-\beta)/\beta$, and $R_T\le T<3(1-\beta)T$.

If $2/3\le\beta<1$, choose within each class a uniformly random baseline set of size $n_c-2d_c$ and two complementary $d_c$-sets. Combining corresponding pieces gives three disjoint nonempty global groups, every pair of which is feasible. Conditional on the baseline, the gate of Appendix~\ref{app:contrast_gate} has expected regret at most $82$ against an outside comparator and nonpositive expected regret against an inside comparator. An arm in class $c$ is outside with probability $2d_c/n_c\le2(1-\beta)$, giving $164(1-\beta)$. At the only integer horizon below $1/\beta$, namely $T=1$, the initial gate includes every arm with probability at least $\beta$. This proves the finite-horizon law. If $\beta=1$, selecting all arms gives zero minimax regret.
\end{proof}
Because the global groups share a single learner, the number of classes does not enter the rate.

\section{Payoff surplus and probe cardinality}
\label{app:mechanisms}
\subsection{Removing the payoff surplus}
\label{app:weak}
For product rewards, define the mean-based weak regret and the surplus of a selected set by
\[
 R_T^{\mathrm{weak},\mu}=\E\sum_{t=1}^T
 \left(\mu^\star-\max_{i\in S_t}\mu_i\right),\qquad
 \sigma(S)=\E\max_{i\in S}X_i-\max_{i\in S}\mu_i.
\]
Then $R_T=R_T^{\mathrm{weak},\mu}-\E\sum_{t\le T}\sigma(S_t)$. This objective extends the utility-based weak regret of \citet[Section~4.4]{chen2017weak}, with utilities $\mu_i$.

For $n=3$ and $k=2$, the minimax mean-based weak regret over independent $[0,1]$-valued arms is at least $\sqrt T/48$, even when all three rewards are observed after every round.

\paragraph{Lower bound.}
Let $\varepsilon=1/(8\sqrt T)$. For a fixed policy, let $P_0$ be the law of its full observable trajectory when all three arms are $\Bern(1/2)$, and let $P_i$ be the corresponding law when only arm $i$ is changed to $\Bern(1/2+\varepsilon)$. Let $O_i$ count the rounds that omit arm $i$, so that $\sum_i\E_0O_i\ge T$. Since the policy kernels are common to both laws, they contribute no divergence, and
\[
 \operatorname{KL}(P_0\Vert P_i)
 =T\operatorname{kl}(1/2,1/2+\varepsilon)
 =-\frac T2\log(1-4\varepsilon^2)
 \le\frac83T\varepsilon^2=\frac1{24}.
\]
Pinsker's inequality~\citep{lattimore2020bandit} and $0\le O_i\le T$ give
\[
 \frac13\sum_i\E_iO_i\ge\frac T3-\frac{T}{4\sqrt3}\ge\frac T6.
\]
The prior-average weak regret is therefore at least $\varepsilon T/6=\sqrt T/48$.

The lower bound isolates the role of the payoff: once the surplus of the maximum is removed, even observing all rewards does not yield bounded regret.

\subsection{Probe cardinality is an information resource}
\label{app:cardinality}
\begin{proposition}[Probe cardinality]
\phantomsection\label{prop:full_budget}
Restrict label-only policies to $|S_t|=k$ in every round. Their product-law minimax value is $\Theta(\delta\min\{T,(n+T)/k\})$, with universal constants.
\end{proposition}
\begin{proof}
Let $V^{\mathrm{lab},=k}_{n,k}(T)$ denote the minimax value when label-only policies must probe exactly $k$ arms in each round. The policy of Appendix~\ref{app:bp_labels} belongs to this class and gives the upper bound. For the learning lower bound, use the ordinary rewards $D_0=\mathrm{Uniform}[a,b]$, where $a=1/2-\varepsilon$ and $b=1/2+\varepsilon$ for $0<\varepsilon<1/2$, and hide an arm with law
\[
 D_q=(1-1/k)\delta_a+(1/k)\delta_1.
\]
In any full set containing $q$, the hidden arm wins exactly in its high branch, which has probability $1/k$, and the ordinary arms are exchangeable. Thus every selected winner is uniform, whether or not $q$ is selected, and the complete adaptive label history is independent of the uniform location of $q$. The hidden mean is $1/2+1/(2k)-\varepsilon(1-1/k)$, and every full action has
\[
 r_q(S)\ge\frac1{2k}\ind\{q\notin S\}-2\varepsilon.
\]
The prior-average regret is therefore at least $\delta T/(2k)-2\varepsilon T$. Letting $\varepsilon\downarrow0$ and combining with coverage gives
\[
 V^{\mathrm{lab},=k}_{n,k}(T)
 \ge\max\{\Phi_{n,k}(T)/2,\delta T/(2k)\}
 \ge\frac\delta4\min\{T,(n+T)/k\}.
\]
Together with \eqref{eq:bp_label_upper}, this proves the claimed order.
\end{proof}
The restriction to full sets matters. In a smaller set of size $2\le s<k$ that contains $q$, the hidden arm still wins with probability $1/k$ rather than $1/s$, and this discrepancy reveals it; a singleton, in contrast, always reports its sole label. Completing an action to full size thus increases its payoff but also changes its ordinal experiment.

\section{The payoff information in opponent identities}
\label{app:contextual}
The soft rank family hides a numerical advantage behind nearly fair comparisons. Under independence, approximate fairness against every reference subset limits this advantage even when the reference marginals differ.

Take a candidate $X\in[0,1]$ and a reference set $R$ of $K-1$ rewards $Y_j\in[0,1]$, all mutually independent. For a reference subset $A$, let
\[
 w(A)=\Prb\{X\text{ is the reported winner against }(Y_j)_{j\in A}\},
 \qquad w(\varnothing)=1,
\]
under the fixed original tie order. Define
\begin{equation}
 \mathfrak U_K^{\mathrm{ctx}}(\eta)
 =\sup\left\{\E X-\E\max_{j\in R}Y_j:
 \left|w(A)-\frac1{|A|+1}\right|\le\eta\quad(A\subseteq R)\right\},
 \label{eq:context_modulus}
\end{equation}
where the supremum ranges over all independent marginal laws of $X$ and the $Y_j$. Let
$L_K(\eta)=\min\{K,1+\log(1/(K\eta))\}$, with $L_K(0)=K$.

\begin{theorem}[Sharp context-conditioned payoff ambiguity]
\phantomsection\label{thm:context_modulus}
For $K\ge2$ and $0\le\eta\le1/K$,
\begin{equation}
 \frac1{256K L_K(\eta)}
 \le\mathfrak U_K^{\mathrm{ctx}}(\eta)
 \le\frac{16384}{K L_K(\eta)}.
 \label{eq:context_sharp}
\end{equation}
The reference rewards may have arbitrary, distinct distributions, including atoms. In particular, exactly fair probabilities on every menu leave payoff ambiguity $\Theta(K^{-2})$.
\end{theorem}
Pairing unequal CDFs turns menu probabilities into a nonnegative multilinear certificate (Lemma~\ref{lem:context_pair}); this bounds the CDF deficit pointwise, and integration gives the payoff bound (Lemma~\ref{lem:context_coverage}). The soft rank family supplies the matching lower bound.

\subsection{Pairing unequal distribution functions}
\label{app:context_pairing}
Partition the reference labels into $g=2(r-1)$ disjoint blocks $R_1,\ldots,R_g$, each containing $a\ge1$ original arms, and put $K=ag+1$. Let $G_1,\ldots,G_g$ be the CDFs of the block maxima, and begin with continuous reference laws. For every nonempty set $A$ of blocks,
\begin{equation}
 w_{\mathcal P}(A)=\Prb\left\{X>\max_{j\in\cup_{i\in A}R_i}Y_j\right\}
 =\E\prod_{i\in A}G_i(X).
 \label{eq:context_moment}
\end{equation}
The empty product is one, matching $w_{\mathcal P}(\varnothing)=1$ for a singleton probe.

Let $m=r-1$ and $\gamma=1/a$, and let $z_1,\ldots,z_m$ be the interior nodes of the right Gauss--Radau rule with $r$ nodes for
\[
 d\sigma_a(z)=\gamma z^{\gamma-1}\,dz,\qquad 0\le z\le1.
\]
Its interior-node polynomial is $Q_m(z)=P_m^{(1,\gamma-1)}(2z-1)$, normalized by $Q_m(1)=r$~\citep[Sections~18.3, 18.5, 18.9]{dlmf}.

\begin{lemma}[A multilinear endpoint certificate]
\phantomsection\label{lem:context_pair}
For any assignment of the $g$ CDFs to $m$ pairs, there is a polynomial $\Psi$ that is multilinear in those CDF coordinates, with
\[
 \Psi(G_1(x),\ldots,G_g(x))\ge0\quad\text{for all }x,
 \qquad \Psi(1,\ldots,1)=1.
\]
Its diagonal $H(z)=\Psi(z,\ldots,z)$ and coefficient norm satisfy
\begin{equation}
 \int H\,d\sigma_a=\Lambda_{a,r}:=\frac1{r\{a(r-1)+1\}},
 \qquad \|\Psi\|_{\mathrm{coeff},1}\le64^{r-1}.
 \label{eq:context_pair_bounds}
\end{equation}
If all moments in \eqref{eq:context_moment} differ from $1/(a|A|+1)$ by at most $\eta$, then
\begin{equation}
 \E\Psi(G(X))\le\Lambda_{a,r}+64^{r-1}\eta.
 \label{eq:context_expectation}
\end{equation}
\end{lemma}
\begin{proof}
For a pair assigned to $z_\ell$, let $t_\ell$ be the earlier of its two $z_\ell$-quantiles, and orient the pair so that $G_i(t_\ell)=z_\ell$ and $G_j(t_\ell)=d_\ell\le z_\ell$; continuity supplies these values even when a CDF has flat intervals. Use the factor
\[
 \frac{(G_i(x)-z_\ell)(G_j(x)-d_\ell)}{(1-z_\ell)(1-d_\ell)}.
\]
Both centered terms are nonpositive below $t_\ell$ and nonnegative above it, so their product is nonnegative. Take $\Psi$ to be the product of these factors over the disjoint pairs. Its diagonal has degree $2m$, vanishes at each $z_\ell$, and equals one at the endpoint, so exactness of the right Radau rule makes its integral equal to the endpoint weight.

The polynomial $Q_m$ is orthogonal to all polynomials of degree below $m$ with respect to $(1-z)z^{\gamma-1}$. Division by $(z-1)Q_m$ proves that the rule is exact through degree $2m$, and applying exactness to squared cardinal polynomials shows that its weights are positive. The endpoint cardinal polynomial is $Q_m/r$, the normalized reproducing kernel at one for degrees up to $m$ under $\sigma_a$. The shifted Jacobi norms give the kernel value $r(r+\gamma-1)/\gamma$, so the endpoint weight is $\gamma/[r(r+\gamma-1)]=\Lambda_{a,r}$.

The coefficient norm is at most
\[
 \prod_{\ell=1}^m\frac{1+z_\ell}{1-z_\ell}\frac{1+d_\ell}{1-d_\ell}
 \le\left|\frac{Q_m(-1)}{Q_m(1)}\right|^2.
\]
The expansion
\[
 Q_m(z)=\sum_{v=0}^m\binom{m+1}{v}\binom{m+\gamma-1}{m-v}
                   (z-1)^{m-v}z^v
\]
has nonnegative binomial coefficients. Vandermonde's identity and $0<\gamma\le1$ bound its coefficient norm, divided by $m+1$, by
\[
 \frac{2^m}{m+1}\binom{2m+\gamma}{m}
 \le\frac{2^{3m+1}}{m+1}\le8^m.
\]
This proves \eqref{eq:context_pair_bounds}. Finally, expanding the multilinear polynomial and using
$\int z^{|A|}d\sigma_a=1/(a|A|+1)$ gives \eqref{eq:context_expectation}.
\end{proof}

\subsection{From the certificate to numerical payoff}
\label{app:context_coverage}
\begin{lemma}[Grouped comparison coverage]
\phantomsection\label{lem:context_coverage}
Suppose $K=2a(r-1)+1\ge256$, $2\le r\le K/128$, and
\[
 \left|w_{\mathcal P}(A)-\frac1{a|A|+1}\right|\le\eta
 \quad(A\subseteq[g]),\qquad 0\le\eta\le1/K.
\]
Then
\begin{equation}
 \E X-\E\max_{j\in R}Y_j
 \le\frac4{Kr}+2\cdot64^{r-1}\eta.
 \label{eq:grouped_coverage}
\end{equation}
The conclusion includes atoms and the original fixed tie rule.
\end{lemma}
\begin{proof}
For continuous references, fix a reward threshold $x$ and put $q=F_X(x)$, $p_j=G_j(x)$, and $P=\prod_jp_j$. It suffices to bound $P-q$ and integrate over $x$. Let
\[
 P_L=2\Lambda_{a,r}=\frac4{r(K+1)},\qquad P_U=1-\frac8K.
\]
If $P\le P_L$, the desired bound is immediate. If $P_L<P\le P_U$, choose the half of the blocks with the largest CDF values. Their product is at least $\sqrt P$, so their menu gives
\[
 (1-q)\sqrt P\le\frac1{a(r-1)+1}+\eta<\frac3K.
\]
It follows that $P-q\le P-1+3/(K\sqrt P)$. The function $P+3/(K\sqrt P)$ is convex in $\log P$; it is at most
$1/128+(3/2)\sqrt{r(K+1)}/K<1$ at $P_L$ and at most $1-4/K$ at $P_U$, so $P-q\le0$ in the middle range. With $4/K$ in place of $3/K$, both endpoint values remain strictly below $1$, which leaves a margin for the limits taken below.

If $P>P_U$, write $e_j=1-p_j$. Then $\sum_je_j\le-\log P\le16/K$. The Radau derivative identity is
\[
 \sum_{\ell=1}^{r-1}\frac1{1-z_\ell}
 =\frac{Q_{r-1}'(1)}{Q_{r-1}(1)}
 =\frac{(r-1)(r+\gamma)}2.
\]
Averaging over assignments of the $2(r-1)$ deficiencies to the paired nodes yields an assignment for which
\begin{equation}
 \sum_\ell\frac{e_{i_\ell}+e_{j_\ell}}{1-z_\ell}
 \le\frac{r+\gamma}{2}\sum_je_j
 \le\frac{8(r+1)}K\le\frac18.
 \label{eq:context_deficiency}
\end{equation}
Orient each pair by its earlier quantile. Since $d_\ell\le z_\ell$, replacing the partner denominator by $1-d_\ell$ only decreases this sum, every normalized factor is positive, and
\[
 \Psi(G(x))=\prod_{\ell=1}^{r-1}
 \left(1-\frac{e_{i_\ell}}{1-z_\ell}\right)
 \left(1-\frac{e_{j_\ell}}{1-d_\ell}\right)\ge\frac78.
\]
This lower bound persists at every candidate value above $x$, so, together with nonnegativity elsewhere, \eqref{eq:context_expectation} implies
\[
 1-q\le\frac87\E\Psi(G(X))
 \le2\Lambda_{a,r}+2\cdot64^{r-1}\eta
 \le\frac4{Kr}+2\cdot64^{r-1}\eta.
\]
Since $P\le1$, this also bounds $P-q$; integration gives the payoff bound for continuous references.

To handle atoms, choose distinct $c_i\in(0,1)$ increasing in the original tie priority and independent uniform variables $U_i$, and replace each reward in the proof by
\[
 X_i^{(\epsilon)}=(1-\epsilon)X_i+
 \epsilon\{(1-\epsilon)c_i+\epsilon U_i\}.
\]
These independent atom-free rewards remain in $[0,1]$. On every fixed menu, the winner converges almost surely to the original tie-broken winner as $\epsilon\downarrow0$, and means and expected maxima converge as well. The finitely many moment errors are at most $\eta+o(1)$ and the middle-menu bound remains below $4/K$, so the preceding argument applies in the limit.
\end{proof}

Averaging in \eqref{eq:context_deficiency} controls the total deficiency and permits $r$ proportional to $K$, whereas coordinatewise worst-case bounds would lose a factor $r$. The certificate uses $4^{r-1}-1$ nonempty grouped menus.

\subsection{The sharp modulus and its endpoint}
\begin{proof}\textbf{of Theorem~\ref{thm:context_modulus}.}
In all cases,
\[
 \E X-\E\max_RY\le\Prb(X>\max_RY)\le w(R)\le1/K+\eta\le2/K.
\]
This proves the upper bound for $K<512$. For $\eta>0$, put $L=1+\log(1/(K\eta))$; the same bound also suffices when $L<8\log64$. For $K\ge512$ and $L\ge8\log64$, take
\[
 r=\min\{\lfloor K/256\rfloor,\lfloor L/(4\log64)\rfloor\}.
\]
Then $r\ge2$ and $r\ge L_K(\eta)/512$. Use $2(r-1)$ disjoint blocks of size $a=\lfloor(K-1)/(2r-2)\rfloor$; their effective budget $K'=2a(r-1)+1$ is at least $K/2$, and $r\le K'/128$. Dropping the leftover references only lowers the reference maximum, so Lemma~\ref{lem:context_coverage} gives
\[
 \E X-\E\max_RY\le\frac8{Kr}+2\cdot64^{r-1}\eta
 \le\frac{16384}{K L_K(\eta)}.
\]
At $\eta=0$, use $r=\lfloor K/256\rfloor$ and omit the error term.

For the lower bound, use i.i.d.\ references and \eqref{eq:soft_rank} with
$r_0=\min\{K,\lceil\sqrt{KL}\rceil\}$ and $p=1/(64r_0^2)$, taking $r_0=K$ at $\eta=0$. If $r_0<K$, every rank-moment error is at most
\[
 pe^{-r_0^2/K}\le e^{-L}/(64KL)\le\eta;
\]
at $r_0=K$, all required moments match exactly. Exchangeability gives the same bound for every menu. Since $p\ge1/[256K L_K(\eta)]$, compressing the ordinary support into $[0,\epsilon]$ and mapping the endpoint rank to reward one gives mean advantage at least $p-\epsilon$. Let $\epsilon\downarrow0$.
\end{proof}

\begin{proposition}[Exact endpoint ceiling]
\phantomsection\label{prop:context_endpoint}
If $w(A)=1/(|A|+1)$ for every subset $A$ of $K-1$ independent references, then the probability that the candidate lies strictly above the essential supremum of every reference is at most
\[
 \frac1{\lfloor(K+1)^2/4\rfloor}.
\]
This bound is attained by the rank-matching law of Lemma~\ref{lem:rank_moments}.
\end{proposition}
\begin{proof}
For $K=2r-1$, apply Lemma~\ref{lem:context_pair} with $a=1$ to all references. The nonnegative certificate equals one above all supports and has expectation $1/r^2$.

For $K=2r$ with $r\ge2$, leave one CDF unpaired, pair the others using the interior roots of $P_{r-1}^{(1,1)}(2z-1)$, which define the right Radau rule for $z\,dz$, and multiply the paired certificate by the unpaired CDF. The expectation of the product is the integral of the paired diagonal against $z\,dz$, and exactness through degree $2r-2$ identifies this integral with the endpoint weight $1/[r(r+1)]$. For $K=2$, use $w(R)=1/2$ directly.

For atoms, use the tie-preserving perturbation above. A candidate strictly above the original supports eventually exceeds all perturbed supports. At fixed $K$, the certificate coefficients are uniformly bounded, so the approximate contextual equalities pass to the limit, and Fatou's lemma gives the bound.
\end{proof}

\paragraph{Why identities matter.}
Fairness averaged over the menus of each size allows larger ambiguity. Let $0<y_1<\cdots<y_{K-1}<\epsilon$ be deterministic references, and put candidate mass $1/K$ in each of the $K-1$ lower rank cells and at reward one. The candidate's rank is uniform, so against a uniform $(s-1)$-subset its winning probability is
\[
 \frac1K\sum_{\ell=0}^{K-1}
 \frac{\binom\ell{s-1}}{\binom{K-1}{s-1}}=\frac1s.
\]
Its mean advantage lies between $1/K-\epsilon$ and $1/K$. Size-averaged fairness thus leaves $\Theta(K^{-1})$ ambiguity, whereas conditioning on opponent identities leaves $\Theta(K^{-2})$.

\end{document}